\PassOptionsToPackage{square,numbers,sort&compress}{natbib}
\documentclass{ecai} 

\usepackage{latexsym}
\usepackage{amssymb}
\usepackage{amsmath}
\usepackage{amsthm}
\usepackage{booktabs}
\usepackage{enumitem}
\usepackage{graphicx}
\usepackage{color}
\usepackage[utf8]{inputenc}
\usepackage{graphics}
\usepackage{cite}
\usepackage{indentfirst}
\usepackage{float}
\usepackage{comment}
\usepackage{bm}
\usepackage{amsfonts}
\usepackage{mathrsfs}
\usepackage{multirow}
\usepackage{diagbox}
\usepackage{tcolorbox}
\usepackage{bbm}
\usepackage{subcaption}
\usepackage{caption}
\usepackage{algorithm}
\usepackage[noend]{algpseudocode}
\usepackage[square,numbers,sort&compress]{natbib} 
\let\cite\citep
\usepackage{etoolbox}
\makeatletter
\patchcmd{\algorithmic}{\itemsep0pt}{\itemsep4pt plus1pt minus1pt}{}{}
\makeatother 
\usepackage{enumitem}
\setlist[itemize]{noitemsep,topsep=0pt,leftmargin=*}

\usepackage{textcomp}
\usepackage{xcolor}

\usepackage{hyperref}

\newtheorem{theorem}{Theorem}
\newtheorem{lemma}{Lemma}
\newtheorem{corollary}{Corollary}

\newtheorem{observation}{Observation}

\newcommand{\red}[1]{{\color{red} #1}}

\newcommand{\ignore}[1]{}

\newcommand{\BibTeX}{B\kern-.05em{\sc i\kern-.025em b}\kern-.08em\TeX}
\newcommand{\bx}{{\mathbf{x}}}
\newcommand{\bC}{{\mathbf{C}}}
\newcommand{\bp}{{\mathbf{p}}}
\newcommand{\bP}{{\mathbf{P}}}
\newcommand{\bmu}{{\boldsymbol{\mu}}}
\newcommand{\bX}{{\mathbf{X}}}

\begin{document}


\begin{frontmatter}


\paperid{8214} 


\title{Online Learning with Probing for Sequential User-Centric Selection}


\author[A]{\fnms{Tianyi}~\snm{Xu}\thanks{Corresponding author. Email: txu9@tulane.edu.}}
\author[B]{\fnms{Yiting}~\snm{Chen}}
\author[A]{\fnms{Henger}~\snm{Li}}
\author[C]{\fnms{Zheyong}~\snm{Bian}}
\author[B]{\fnms{Emiliano}~\snm{Dall'Anese}}
\author[A]{\fnms{Zizhan}~\snm{Zheng}}

\address[A]{Tulane University, United States}
\address[B]{Boston University, United States}
\address[C]{University of Houston, United States}


\begin{abstract}
We formalize sequential decision–making with information acquisition as the Probing‑augmented User‑Centric Selection (PUCS) framework, where a learner first probes a subset of arms to obtain side information on resources and rewards, and then assigns $K$ plays to $M$ arms.  PUCS encompasses practical scenarios such as ridesharing, wireless scheduling, and content recommendation, in which both resources and payoffs are initially unknown and probing incurs cost. For the offline setting (known distributions) we present a greedy probing algorithm with a constant‑factor approximation guarantee of $\zeta=(e-1)/(2e-1)$.  For the online setting (unknown distributions) we introduce \textsc{OLPA}, a stochastic combinatorial bandit algorithm that achieves a regret bound of $\mathcal{O}\!\bigl(\sqrt{T}+\ln^{2}\!T\bigr)$.  We also prove an $\Omega(\sqrt{T})$ lower bound, showing that the upper bound is tight up to logarithmic factors. Numerical results using the dataset collected from the real world demonstrate the effectiveness of our solutions.
\end{abstract}

\end{frontmatter}

\section{Introduction}

Stochastic multi‑armed bandits (MAB)~\cite{lai1985asymptotically} and Multi‑player MABs (MP‑MAB)~\cite{anantharam1987asymptotically} have been extensively studied as general frameworks for sequential decision making, where the goal is to maximize the cumulative reward by sequentially choosing from a set of options, or ``arms''~\cite{lai1985asymptotically,auer2002finite}. These models have been applied in a variety of domains, including personalized recommendation~\cite{li2010contextual}, financial portfolio management~\cite{huo2017risk}, and dynamic resource allocation~\cite{anantharam1987asymptotically}. 

To cope with complex real‑world scenarios, the user‑centric selection problem has recently been introduced as an extension of MP‑MABs \cite{chen2022online}. Consider context recommendation as an example. The decision‑maker recommends multiple pieces of content (mapped to arms in MP‑MABs) and selects a subset of slots (plays) to allocate to these arms (assignment decision). Unlike traditional MP‑MABs, the user‑centric selection problem allows each arm to be associated with stochastic units of resources, and multiple plays can pull the same arm. Applications of the user‑centric selection problem are vast and include personalized content delivery \cite{li2010contextual} and resource allocation in ridesharing platforms \cite{chen2022online}.

As another example, consider the user‑centric selection problem in ridesharing~\cite{chen2022online}. Pickup locations correspond to arms, and drivers can be modeled as plays. There is a moving cost to a pickup location that is related to the reward. Multiple drivers can drive to the same location, but there is a limit on the maximum number of requests (resources) that can be hosted together at any location. This example clearly illustrates the user‑centric selection problem, but in practice, some information is unknown. For instance, in this example, the distribution of passenger requests in a certain area is unknown. The real‑time road conditions are also unknown. These unknown factors can significantly impact the accuracy of decision‑making. The classic MAB framework relies purely on instantaneous feedback received after each decision round to obtain a desired exploration vs.\ exploitation trade‑off. 

In such an uncertain environment, probing is a promising approach for acquiring additional information when searching for the best alternative under uncertainty. It was initially studied in economics~\cite{pandora-Weitzman} and has found applications in database query optimisation~\cite{munagala2005pipelined,deshpande2016approximation,liu2008near}. Recently, it has been used to obtain real–time road traffic conditions (subject to a cost) before making vehicle‑routing decisions~\cite{bhaskara2020adaptive}, and for collecting link‑quality information for scheduling in wireless access points~\cite{xu2021joint,xu2023online}. 

In this work, we propose a unified framework that integrates probing strategies with assignment decisions in the context of user‑centric selection. We term this abstraction the Probing‑augmented User‑Centric Selection (PUCS) framework.  PUCS explicitly couples the probing decision with the subsequent play‑to‑arm assignment, enabling us to quantify both the value \emph{and} the cost of information acquisition.  Concretely, in every round the decision maker first probes a budget‑limited subset of arms to observe their resources and rewards, and then assigns the $K$ plays to the $M$ arms accordingly. In the offline case where the reward and resource distributions are known a~priori, we derive a greedy probing algorithm that achieves a constant approximation factor by exploiting the submodularity of the objective function. For the more challenging online setting, we develop a combinatorial bandit algorithm and prove a regret bound of $\mathcal{O}\!\bigl(\sqrt{T}+\ln^{2}\!T\bigr)$. 

We remark that \cite{zuo2020observe} also considers probing in a similar multi‑player MAB setting. However, they do not consider multiple resources for arms and assume that two or more plays assigned to the same arm can lead to a collision. In addition, they assume the rewards follow a Bernoulli distribution, while we consider general distributions. Further, they do not consider a given probing budget as we do.

\ignore{
The primary contributions of this work are as follows:
\begin{itemize}[noitemsep,topsep=0pt,parsep=0pt]
\item We develop the PUCS framework for joint probing and assignment in user‑centric selection problems.
\item In the offline setting with known distributions, we design a greedy probing algorithm with a constant‑factor approximation guarantee \(\zeta=(e-1)/(2e-1)\).
\item In the online setting with unknown distributions, we devise a stochastic combinatorial bandit algorithm (\textsc{OLPA}) and prove a regret bound \(\mathcal{O}\!\bigl(\sqrt{T}+\ln^{2}T\bigr)\) together with an \(\Omega(\sqrt{T})\) lower bound.
\item Extensive experiments on real‑world datasets demonstrate the empirical effectiveness of our approach over strong baselines.
\end{itemize}
}

\paragraph{Contributions.}The primary contributions of this work are as follows:
\begin{itemize}[leftmargin=*,labelsep=0.4em,itemsep=0pt,topsep=0pt,parsep=0pt]
    \item We introduce the \emph{PUCS} framework for jointly choosing which arms to probe and how to assign plays in user‑centric selection problems.
    \item For the offline setting with known distributions, we design a greedy probing algorithm that admits a constant‑factor approximation guarantee $\zeta=(e-1)/(2e-1)$.
    \item For the online setting with unknown distributions, we propose a two‑phase stochastic combinatorial bandit algorithm (\textsc{OLPA}) and establish a regret bound $\mathcal{O}\!\bigl(\sqrt{T}+\ln^{2}T\bigr)$ together with an $\Omega(\sqrt{T})$ lower bound.
    \item Extensive experiments on real‑world datasets demonstrate the empirical effectiveness of our approach over strong baselines.
\end{itemize}

\section{Related Work}

Sequential decision-making and online learning have been extensively studied in various contexts. Classical approaches such as multi-armed bandits \cite{auer2002finite,bubeck2012regret} provide foundational strategies for balancing exploration and exploitation. Recent advancements have extended these models to more complex scenarios, including combinatorial bandits \cite{chen2013combinatorial}, and multi-player MAB \cite{komiyama2015optimal} which are more relevant to our sequential user-centric selection problems.

Sequential user-centric selection problems can be viewed as a variant and extension of the multi-player multi-armed bandit (MP-MAB) framework. MP-MAB models are designed to handle situations where multiple players compete for shared arms, requiring strategies that account for both coordination and competition among players \cite{zhou2018budget,mo2023multi}. These models are particularly effective in scenarios like recommendation systems or resource allocation tasks, where multiple users interact with a common set of options. Unlike traditional MP-MABs, the user-centric selection problem allows each arm to be associated with stochastic units of resources and multiple plays can pull the same arm \cite{chen2022online}.

A ridesharing example is considered in \cite{chen2022online} about the sequential user-centric selection problem. In this scenario, pickup locations can be mapped to arms, and drivers can be modeled as plays. The cost of moving to a pickup location is related to the reward, and the ride requests arriving at each arm can be modeled as the resource. Multiple drivers can choose to drive to the same location. However, this approach overlooks the unknown information of the real world, such as real-time traffic information, which can significantly impact decision-making. 

To address these unknown factors in real‑world systems, one can actively probe
the environment and acquire side information before acting.  Selecting which
items to probe is, however, computationally intractable in general—the
underlying decision problems are typically NP‑hard \cite{goel2006asking}.
Notably, the work of Golovin and Krause on adaptive submodularity \cite{golovin2011adaptive} has provided theoretical guarantees for greedy algorithms in adaptive settings. This insight has been used
in diverse applications, including active learning \cite{guillory2010interactive}
and network monitoring \cite{leskovec2007cost}, where limited probing gathers
critical information before decisions are made.


Recent studies~\cite{bhaskara2020adaptive}, have considered probing strategies at a certain cost to obtain information about multiple road conditions. However, they primarily focus on the probing itself without considering how to make decisions afterward. \cite{xu2021joint,xu2023online} explored the use of probing to gather additional information before making decisions in the context of a wireless network access point serving users. However, their work is limited to a single-player setting and is only applied in the context of wireless networks. \cite{zuo2020observe} proposed a general framework that incorporates observation (probing) into MP-MABs, but they only considered probing rewards from a Bernoulli distribution and could not apply their approach to sequential user-centric selection problems. In contrast, we introduce the PUCS framework, which jointly models probing and assignment in sequential user‑centric selection problems, and we consider general distributions.  Within this framework we design both offline and online algorithms and provide theoretical guarantees for each. 




\section{Probing‑Augmented User‑Centric Selection (PUCS) Framework}

In this section we formalise the Probing‑augmented User‑Centric Selection (PUCS) problem.  We first describe the sequential user–centric model without probing and then show how probing is integrated on top of it.


\subsection{Arm, Play, and Resource Model}
Consider a sequential user-centric decision problem~\cite{chen2022online} over \( T \in \mathbb{N}^+ \) time slots, which is a variant of the classic multi-play multi-armed bandits problem. There are a set of \( M \in \mathbb{N}^+ \) arms and a set of \( K \in \mathbb{N}^+ \) plays. In each time slot, the decision is to assign the \( K \) plays denoted by \([K]:=\{1, 2, \ldots, K\}\) to the arm set denoted by \([M]:=\{1, 2, \ldots, M\}\). Each play can only be assigned to one arm and different plays can be assigned to the same arm.

Each arm \( m \in [M] \) has $D_{t,m}$ units of resource in time slot $t$, where $D_{t,m}$ is an i.i.d. sample from a distribution with support \(\{1, 2, \ldots, D_{\max}\}\). Let \( \bp_m = [p_{m,d}: \forall d \in \{1, 2, \ldots, D_{\max}\}] \) denote the probability mass function of \( D_{t,m} \), where \( p_{m,d} = \mathbb{P}[D_{t,m} = d] \). We define a probability mass matrix $\bP\in\mathbb{R}^{M\times D_{\max}}$ with $\bP_{m,d}=p_{m,d}$. 

Each play assigned to an arm consumes one unit of resource for that arm. If the play $k$ gets one unit of resource from arm \( m \) in time slot \( t \), it receives a reward \( R_{t,m,k} \), again an i.i.d. sample from an unknown distribution. We denote the expectation of $R_{t,m,k}$ as:
$$
\mu_{m,k}=\mathbb{E}\left[R_{t, m,k}\right], \quad \forall t \in[T], m \in[M], k\in[K] .
$$
We define the reward mean matrix as $\bmu\in\mathbb{R}^{M\times K}$ with $\bmu_{m,k}=\mu_{m,k}$. We denote the $m$-th row of $\bmu$ by $\bmu_m$ and hence $\bmu_m=\begin{bmatrix}
   \mu_{m,1} & ... &  \mu_{m,K} 
\end{bmatrix}.$
We further let \(\{F_{m,k}\}_{m\in[M],k\in[K]}\) denote the cumulative distribution function (CDF) of rewards for each arm-play pair. As shown below, \(\{F_{m,k}\}\) (rather than just \(\{\bmu_{m,k}\}\)) are needed to evaluate the objective functions, because the objective functions depend on the entire distribution of rewards, not just their expected values.

Thus, the resource and reward associated with arm \( m \in \{1, 2, \ldots, M\} \) is characterized by a list of 
the random vectors \((D_m, R_{m,1},...,R_{m,K})\), where \( D_m = [D_{t,m} : \forall t \in \{1, 2, \ldots, T\}] \) and \( R_{m,k} = [R_{t,m,k} : \forall t \in \{1, 2, \ldots, T\}] \), $\forall k\in[K]$. We assume that  \( D_{t,m} \) and \( R_{t,m,k} \), $\forall m \in[M], \forall k\in[K]$,  are  i.i.d. from each other and across time slots. 

Let \( C_{t,m} \subseteq [K] \) denote the collection of plays that pull arm $m$ in time slot $t$.  The action profile of all plays is denoted by \( \bC_t:=\{C_{t,m}:~m\in[M] \}  \). 
Let $|C_{t,m}|$ denote the number of plays assigned to arm $m$ in time $t$. 

Below we give two examples of the above model.
\begin{itemize}
\item Ridesharing Services: Plays can be mapped to cars, arms to pickup locations, resources to passengers, and rewards to fares. The decision maker must decide which cars to send to which locations and multiple cars can be sent to the same location.

\item Content Recommendation: Plays can be mapped to recommendation slots, arms to different content pieces, resources to user engagement (e.g., clicks or views), and rewards to the engagement level (e.g., watch time or read depth). The decision maker needs to assign content to recommendation slots and the same piece of content can be assigned to multiple slots.
\end{itemize}

\subsection{Probing Model} 

In the original user-centric selection problem~\cite{chen2022online}, similar to the classic multi-armed bandits setting, the instantaneous reward associated with an arm and the number of resource units available at an arm is observed only when the arm is played (i.e., assigned to a play). A key observation of this paper is that probing can often be utilized to gather additional on‑demand information in practice. For example, in ridesharing, probing can be used to check real-time traffic conditions~\cite{bhaskara2020adaptive} or get more accurate information about passenger availability and fare rates. Similarly, in content recommendation, probing in the form of A/B tests or surveys can be used to gauge user interest before making broader recommendations.

Probing differs fundamentally from prediction: it measures the current state directly (e.g.\ actual queue length, current link quality), avoiding model‑drift errors that accumulate in time‑varying systems.  However, probes consume limited system resources such as driver detours, sensing bandwidth, or user attention.  We therefore assume a per‑round budget that restricts the learner to probe at most \(I\) arms, modelling the realistic “sense‑some‑but‑not‑all” constraint used in prior work on adaptive sensing and opportunistic measurement~\cite{bhaskara2020adaptive,xu2023online}.  Because the platform (e.g.\ TNC backend, recommender engine, WLAN controller) controls where to send drivers, which articles to A/B‑test, or which channels to sense, it can decide which subset \(S_t\) of arms to probe in each round, consistent with existing deployments.

To this end, we assume that in each time slot $t$, the decision maker can probe a subset of arms $S_t \subseteq [M]$ before the arm-play assignment is made. For each probed arm $m \in S_t$, the decision maker immediately observes the realization of both $D_{t,m}$ and $R_{t,m,k}$, for all $k \in [K]$. This is a reasonable assumption in practice. For example, a single probe at a pickup location reveals its current vacancy and provides an estimate of the nearby traffic conditions, which determines the pickup cost.


Let $N_{t,m}$ denote the realization of $D_{t,m}$, and $X_{t,m,k}$ the realization of \( R_{t,m,k} \), and let \( \bX_{t,m} := [X_{t,m,k} : \forall k \in[K] ] \). In each time slot, the decision maker first probes a subset of arms $S_t\subseteq [M]$ and based on the values of $N_{t,m}$ and $\bX_{t,m}$, $\forall~m\in S_t$, the decision maker assigns $K$ plays to $M$ arms and each play gets a reward if it receives one unit of resource from the assigned arm.



We then derive the total rewards from an arm $m$ by distinguishing two cases. First, assume that $m$ is in the probing set $S_t$. If play $k$ receives one unit of resource from arm \( m \), the reward received by play $k$ is the same as the probed value, which is $X_{t,m,k}$. That is, we assume that probing is accurate and the random variable $R_{t,m,k}$ stays the same within a time round. Similarly, we assume the random variable $D_{t,m}$ stays the same within a time round. That is, the decision maker observes the actual amount of resources, $N_{t,m}$, through probing if $m\in S_t$. 
If \( N_{t,m} \geq |C_{t,m}| \), each play receives one unit of resource and the total rewards returned by arm $m$ is $\sum_{k\in C_{t,m}} X_{t,m,k}$. If \( N_{t,m} < |C_{t,m}| \), only \( N_{t,m} \) plays get resource, and the remaining \(  |C_{t,m}| - N_{t,m} \) plays in $C_{t,m}$ do not get resources. 
We further assume that the plays associated with larger rewards are given priority. Formally, let  $\bX_{t,m}({C_{t,m}}):=\{\bX_{t,m,k}: k\in C_{t,m}\}$ and let $\bX_{t,m}^{\text{sort}}$ be the sequence generated by sorting $\bX_{t,m}({C_{t,m}})$ from largest to smallest. Then the total rewards returned by arm $m$ is  $\sum_{1 \leq i \leq N_{t,m}}\bX_{t,m,i}^{\text{sort}}$ if  \( N_{t,m} < |C_{t,m}| \). To summarize, if $m\in S_t$, the total rewards returned by arm $m$ is
\begin{align*}    
&\mathcal{R}_{m}^{\text{prob}}(C_{t,m}; \bX_{t,m},N_{t,m}) =\sum_{i=1}^j X_{t,m,i}^{\text{sort}}\\
&~~~~~~~~~~~~~j:=\min\{N_{t,m}, |C_{t,m}| \}
\end{align*}

We then consider the case when arm $m$ is not in the probing set $S_t$. In this case, if play $k$ receives one unit of resource from arm \( m \), the reward received by play $k$ is a sample of $R_{t,m,k}$. Further, the value of $D_{t,m}$ (i.e., the actual number of resources of arm $m$ in time slot $t$) is unknown, although we assume it is fixed 
within a time round.  If \( D_{t,m} \geq |C_{t,m}| \), each play receives one unit of resource, and the expected total rewards returned by arm $m$ is $\sum_{k\in C_{t,m}} \mu_{m,k}$. If \( D_{t,m} < |C_{t,m}| \), only \( D_{t,m} \) units of resource are allocated, and \(  |C_{t,m}| - D_{t,m} \) plays in $C_{t,m}$ do not get resources. In this case, we assume that the plays associated with larger \textbf{expected} rewards are given priority. 
Formally,  let $\bmu_{m}({C_{t,m}}):=\{\mu_{m,k}: k\in C_{t,m}\}$ and $\bmu_{m}^{\text{sort}}$ be the sequence generated by sorting $\bmu_{m}({C_{t,m}})$ from largest to smallest. Then the expected total rewards returned by arm $m$ is  $\sum_{1 \leq i \leq D_{t,m}}\bmu_{m,i}^{\text{sort}}$ if  \( D_{t,m} < |C_{t,m}| \).  From the assumed independence between $D_{t,m}$ and $(R_{t,m,1},...,R_{t,m,K})$, 
it can be shown that the total expected reward returned by arm $m$ is 
\begin{align*}    
&\mathcal{R}_m(C_{t,m}; \bmu_{m}, \bp_m)\\
:=&\mathbb{E} \left[\sum_{i=1}^j \bmu_{m,i}^{\text{sort}}\right],~j:=\min\{D_{t,m}, |C_{t,m}| \} \\
= & \sum_{d=1}^{ |C_{t,m}| } \left( p_{m,d}\sum_{i=1}^d \bmu_{m,i}^{\text{sort}}\right) + \sum_{d=|C_{t,m}|+1}^{D_{\max}}  p_{m,d}\left(\sum_{i=1}^{|C_{t,m}|} \bmu_{m,i}^{\text{sort}}\right). 
\end{align*}
by adapting a similar argument in~\cite{chen2022online}. A detailed proof is given in the supplementary material. In the special case that $\mu_{m,k_1}=\mu_{m,k_2}$ for any $k_1,k_2\in [K]$, the above definition of $\mathcal{R}_m(C_{t,m}; \bmu_{m}, \bp_m)$ reduces to the definition in \cite{chen2022online}.

To model the probing overhead, we further assume that up to $I$ arms can be probed in a single time round and the final total rewards obtained for probing $S_t$ and action profile $\bC_t$ is 
\begin{align*}
    \mathcal{R}_t^{\text{total}}=\left(1 - \alpha( |S_t|) \right)\left(\sum_{m\in S_t} \mathcal{R}_{m}^{\text{prob}} + \sum_{m\in [M]\setminus S_t}\mathcal{R}_{m}\right)
\end{align*}
where $\alpha:\{0,1,...,I\}\mapsto [0,1]$ is a non-decreasing function and $\alpha(0)=0$, $\alpha(I)=1$. Here $\alpha(S)$ captures the probing overhead as the percentage of reward loss, which is monotonic with the size of $S$. 
This is appropriate because the probing impact often scales with the system's overall performance. In many scenarios, the opportunity cost of probing - such as time, resources, or energy - depends on the potential rewards that could have been achieved without probing. A similar approach has been adopted in~\cite{zuo2020observe,xu2023online}.

\subsection{Objective and Regret}
\ignore{
To begin with, We note that the total reward \(\mathcal{R}_t^{\text{total}}\) depends on:
\begin{itemize}
    \item the probing set \( S_t \);
\item Realizations \( X_{t,m,k} \) for \( m \in S \) and \( k \in [K] \);
\item Realizations \( N_{t,m} \);
\item Probability \( \bp_m \) for \( m \notin S \);
\item Expected rewards \(\mu_{m,k}\) for \( m \notin S \) and \( k \in [K] \);
\item The action profile $\bC_t$, which determines the assignment policy \( C_{t,m} \) for each arm $m$.
\end{itemize}

\red{in the following, we omit $t$}

From the ob
Given probing set $S$ and realization of reward $\bX_{m}$ and number of resource $N_{m}$, we define the optimal action profile $\bC^*(S,\bX_{m},N_{m})$ to be such that maximize total reward \(\mathcal{R}_t^{\text{total}}\). Under the optimal action profile $\bC^*(S,\bX_{m},N_{m})$, 
}

We then define the objective of our Probing‑augmented User‑Centric Selection (PUCS) problem. Recall that in each time slot $t$, the decision maker first picks a subset of arms $S_t$ to probe and observes the instantaneous rewards and amount of resources for each arm in the probed set. It then determines the set of plays $C_{t,m}$ assigned to each arm $m$. 

We first consider the offline case where $\bp_m$ and $\bmu_m$ are known a priori. In this case, it suffices to consider each time slot separately and we drop the subscript $t$ to simplify the notation. We first observe that when the probing set $S$ is fixed and the observations from probing are given, the optimal assignment can be obtained. Formally,  let $h_{\text{total}}(S,\{N_m,\bX_m\}_{m}, \{\bp_m,\bmu_m\}_{m\notin S})$ denote the maximum expected reward in a time slot with probing set $S$ and the given realizations of rewards and resources. That is, 


\begin{equation}
\begin{aligned}
&h_{\text{total}}(S,\{N_m,\bX_m\}_{m\in S}, \{\bp_m,\bmu_m\}_{m\notin S}) \\
:= & \max\limits_{\substack{C_{m}\subseteq [K] \\ m\in [M]}} \left( \sum_{m \in M\setminus S} \mathcal{R}_m(C_{m}; \bmu_m, \bp_m) \right. \\
& \phantom{ \max\limits_{\substack{C_{m}\subseteq [K] \\ m\in [M]}} } \left. + \sum_{m \in S} \mathcal{R}_{m}^{\text{prob}}(C_{m}; \bX_{m},N_{m}) \right) \\
&\text{s.t.} \quad C_{m_1} \cap C_{m_2} = \emptyset, \,\forall m_1 \!\neq\! m_2, \, m_1, m_2 \in [M]
\end{aligned}
\end{equation}

\noindent where the constraints ensure that each play is assigned to at most one arm. This problem can be solved by adapting Algorithm 1 in~\cite{chen2022online} that considers the special case when $S$ is empty. This is further elaborated in Observation 1 in the supplementary material.





From the above discussion, 
the joint probing and assignment problem simplifies 
to finding an optimal probing set $S$. Let $f(S)$ denote the 
the optimal total expected reward that can be obtained for a given probing set $S$, that is


\begin{align*}
f(S):=\mathbb{E}_{ \substack{X_{t,m,k}\sim R_{t,m,k} \\ N_{t,m}\sim D_{t,m}\\ m\in S,~i\in[K] } } 
& \, h_{\text{total}}\left(S, \{N_m,\bX_m\}_{m}, \right.\\
& \, \hspace{1.2cm} \left.\{\bp_m,\bmu_m\}_{m\notin S}\right).
\end{align*}
Then the offline problem is  
to finding a probing set $S$ to maximize $R(S):=(1-\alpha(|S|))f(S)$, that is, 
\begin{equation}\label{eq:optimal_probing}
      S^*:= \arg\max_{S\subseteq [M]} R(S)
\end{equation}
Finding the optimal solution to this problem is hard for general reward and resource distributions. Fortunately, there is an efficient approximation algorithm with a constant approximation factor as we show in the next section by exploring the structure of $f(S)$. Specifically, we say that an offline algorithm is an \(\zeta\)-approximation $(0 < \zeta\ \leq 1)$ if it achieves an expected $R(S_t)$ of at least \(\zeta R(S^*_t)\).



In the more challenging online setting with unknown $\bp_m$ and $\bmu_m$, we quantify the performance of a probing policy via the regret \( \mathcal{R}_{\text{regret}}(T) \), which is defined as the difference between the cumulative reward of the optimal probing policy and the cumulative reward of our probing strategy. 
In this scenario, it is challenging to achieve sublinear regret with respect to the optimal $S^*_t$ as commonly considered in the MAB literature. Instead, our objective is to achieve sublinear \(\zeta\)-approximation regret, defined as follows~\cite{chen2016combinatorial}:
\begin{align}
\mathcal{R}_{\text{regret}}(\zeta,T) = & \sum_{t=1}^{T} \zeta  R(S^*_t) - \sum_{t=1}^{T} R(S_t)
\label{regret}
\end{align}

\section{The Offline Setting}\label{sec:offline}

In the offline PUCS setting, the probability mass matrix $\bP$ and the true CDF $\{F_{m,k} \}_{m\in[M],k\in[K]}$ of rewards for each arm-play pair are known. 
We aim to determine the optimal probing set that maximizes the expected total reward, which is in the equation (\ref{eq:optimal_probing}). 
This maximization problem is computationally challenging; in fact, similar problems in related settings have been shown to be NP-hard in prior work \cite{goel2006asking}. Consequently,
 we design a surrogate objective function $f_{\text{prob}}(S)$ that captures the marginal value of probing. We further prove that this surrogate is both monotonic and submodular.   Leveraging these properties, a greedy algorithm (Algorithm~\ref{offline_greedy})
provably achieves a constant‑factor approximation to the optimal offline
solution.

\subsection{Offline Greedy Probing}
For small problems, the optimal solution for the offline setting can be obtained by solving Problem~\eqref{eq:optimal_probing} with an exhaustive search. 
For larger instances, however, a more efficient solution is needed. In this section, we derive a greedy algorithm that provides a constant factor approximation. 
One work \cite{nemhauser1978analysis} states that a simple greedy algorithm with an objective function that is monotone and submodular can obtain an approximation factor of an approximation factor of $1-1/e$. Based on this, we construct our objective function $f_{\text{prob}}(S)$ that is monotone and submodular shown in Lemma \ref{lemma3} and Lemma \ref{lemma4} and design an Offline Greedy Probing algorithm to obtain the approximation solution. 

Let $h_{\text{prob}}$ denote the value of the optimal assignment 
for the special case where all the plays can only select the probed arms. 
Formally,  given a probing set $S$ and the realization of rewards $\bX_m$ and number of resources $N_{m}$, $h_{\text{prob}} (S,\{\bX_m,N_{m}\}_{m\in S})$ is defined as

\begin{equation}
\begin{aligned}
&h_{\text{prob}}:= 
          \max\limits_{\substack{ C_{m}\subseteq [K],  \\ m\in S} }     
   \left(\sum_{m\in S} \mathcal{R}_{m}^{\text{prob}}(C_{m}; \bX_{m},N_{m}) \right)\\
&\text { s.t. }~  C_{m_1}\cap C_{m_2}=\emptyset, ~\forall m_1\neq m_2,   m_1,m_2\in S
\end{aligned}
\end{equation}

$h_{\text{prob}}$ has an optimal assignment policy based on Observation 1 in the supplementary material. We then define 
\[ f_{\text{prob}}(S):=\mathbb{E}_{\substack{X_{m,k}\sim R_{m,k} \\ N_{m}\sim D_{m}\\ m\in S,~k\in[K] }   } ~h_{\text{prob}} (S,\{ \bX_m, N_{m}\}_{m\in S}).  \]
$f_{\text{prob}}(S)$ is monotone and submodular shown in Lemma 3 and Lemma 4, which is used for proving Theorem 1.

Similarly, we use $f_{\text{unprobed}}$ to denote the value of the optimal assignment for the case where all the plays can only select the unprobed arms. Formally,  given a probing set $S$ , we have 
\begin{equation}
\begin{aligned}
&f_{\text{unprobed}}(S):= 
          \max\limits_{\substack{ C_{m}\subseteq [K],  \\ m\in [M]\setminus S}}   
   \left(\sum_{m\in M\setminus S}\mathcal{R}_m(C_{m}; \bmu_m, \bp_m) \right)\\
&\text { s.t. }~  C_{m_1}\cap C_{m_2}=\emptyset, ~\forall m_1\neq m_2,   m_1,m_2\in  [M]\setminus S
\end{aligned}
\end{equation}

We note that if $S=\emptyset$, $R(S)=f(S)=f_{\text{unprobed}}(S)$. With these notations, we are ready to explain our greedy probing algorithm for the offline setting (Algorithm \ref{offline_greedy}). It starts by initializing \( I \) empty sets \( S_i \) for \( i = 0 \) to \( I-1 \), where \( I \) is the probing budget (lines 1-2). In each iteration from 1 to \( I-1 \), the algorithm selects the arm \( m \) that, when added to the current set \( S_{i-1} \), provides the maximum marginal increase in the expected reward function \( f_{\text{prob}} \) (line 4). The algorithm then updates the probing set \( S_i \) to include the newly selected arm (line 5). It determines the optimal probing set \( S^{\rm pr}  \) by selecting the set that maximizes the adjusted reward function \((1-\alpha(i)) f_{\text{prob}}(S_i)\), where \(\alpha(i)\) captures the probing cost as a function of the set size (lines 6-7).
If the adjusted reward for the selected probing set \( S_j \) is less than the reward obtained without probing, \( f_{\text{unprobed}}(\emptyset) \), then the final probing set \( S^{\rm pr} \) is set to be empty (lines 8-9).

To show that the greedy algorithm is nearly optimal, we first establish the following properties. The proofs are given in the supplementary material.

\begin{lemma}
Given a probing set $S$, it holds that 
\[f(S)\leq f_{\text{prob}}(S)+ f_{\text{unprobed}}(S).\]
\label{lemma1}
\end{lemma}

\begin{lemma}
$f_{\text{unprobed}}(S)$ is monotonically decreasing, i.e., for any $S\subseteq T\subseteq [M]$,  $f_{\text{unprobed}}(S)\geq  f_{\text{unprobed}}(T) $. 
\label{lemma2}
\end{lemma}

\begin{lemma}
$f_{\text{prob}}(S)$ is monotonically increasing, i.e., for any $S\subseteq T\subseteq [M]$,  $f_{\text{prob}}(S)\leq  f_{\text{prob}}(T) $. 
\label{lemma3}
\end{lemma}

\begin{lemma}
$f_{\text{prob}}(S)$ is submodular.
\label{lemma4}
\end{lemma}

\begin{lemma}
Let $S_i$ be the $i$-th probing set found by Algorithm \ref{offline_greedy} line 5 , $\tilde{S}^*$ be the one that maximizes $(1-\alpha(|S_i|)) f\left(S_i\right)$ (Algorithm \ref{offline_greedy}\ line 6),  $S^{\text {pr }}$ be the final output of Algorithm \ref{offline_greedy} lines 7-9.  Then it holds that  $f_{\text{unprobed}}(\emptyset)\leq R(S^{\text{pr}})$ and $(1-\alpha(|\tilde{S}^*|))  f_{\text{prob}}(\tilde{S}^*) \leq R(S^{\text{pr}})$.
\label{lemma5}
\end{lemma}
\begin{theorem}
Let  $S^*:= \arg\max_{S\subseteq [M]} R(S)$. Algorithm \ref{offline_greedy} outputs a subset $S^{\text {pr }}$ such that $R\left(S^{\text {pr }}\right) \geq \zeta R\left(S^{\star}\right)$ where $\zeta=\frac{e-1}{2 e-1}$.
\label{offalpha1}
\end{theorem}
\begin{proof}
Let $S_i$ be the $i$-th probing set found by Algorithm \ref{offline_greedy} line 5 , $\tilde{S}^*$ be the one that maximizes $(1-\alpha(|S_i|)) f\left(S_i\right)$ (Algorithm \ref{offline_greedy}\ line 6). We have
$$
\begin{aligned}
R\left(S^{\star}\right) & =(1-\alpha(|S^*|)) f(S^*)\\
& \leq(1-\alpha(|S^*|)) (f_{\text{prob}}(S^*)+ f_{\text{unprobed}}(S^*))\\
&\leq (1-\alpha(|S^*|)) f_{\text{prob}}(S^*)+f_{\text{unprobed}}(S^*)\\
&\leq (1-\alpha(|S^*|)) f_{\text{prob}}(S^*)+f_{\text{unprobed}}(\emptyset)\\
&\leq  (1-\alpha(|S^*|)) \frac{e}{e-1} f_{\text{prob}}(S_{|S^*|})+f_{\text{unprobed}}(\emptyset)\\
&=  (1-\alpha(|S_{|S^*|}|)) \frac{e}{e-1} f_{\text{prob}}(S_{|S^*|})+f_{\text{unprobed}}(\emptyset)\\
&\leq  (1-\alpha(|\tilde{S}^*|)) \frac{e}{e-1} f_{\text{prob}}(\tilde{S}^*)+f_{\text{unprobed}}(\emptyset)\\
&\leq \frac{e}{e-1}  R\left(S^{\text {pr }}\right)+ R\left(S^{\text {pr }}\right)\\
&=\frac{2e-1}{e-1}  R\left(S^{\text {pr }}\right)
\end{aligned}
$$
The first inequality follows by Lemma \ref{lemma1}; the second inequality follows by $1-\alpha(|S^*|)\leq 1$;  the third inequality follows by Lemma \ref{lemma2}; the fourth inequality follows by \cite{nemhauser1978analysis} and Lemma \ref{lemma3} and  \ref{lemma4} ; the fifth inequality by the fact that  $\tilde{S}^*$  maximizes $(1-\alpha(|S_i|)) f\left(S_i\right)$; the last inequality follows by Lemma \ref{lemma5}.
\end{proof}

Lemma \ref{lemma4} and Theorem \ref{offalpha1} are our main contributions.
In conclusion, our algorithm effectively balances the trade-off between the benefit of probing more arms and the associated costs, ensuring that the selected probing strategy obtains nearly optimal reward in the offline setting.

\begin{algorithm}[!t]\label{alg:offline_nonadaptive_probing} 
  \caption{Offline Greedy Probing} 
  \label{offline_greedy}
  \begin{algorithmic}[1]
    \Require $\bP$, $\{F_{m,k} \}_{m\in[M],k\in[K]}$
    \Ensure $S^{\rm pr}$: the probing set
    \For {$i=0$ to $I-1$}
        \State $S_i \leftarrow \emptyset$
    \EndFor
    \For {$i=1$ to $I-1$}
\State $m \leftarrow \mathop{\rm argmax}\limits_{m \in [M] \backslash S_{i-1}}{(f_{\text{prob}}(S_{i-1} \cup \{m\}) - f_{\text{prob}}(S_{i-1}))}$
       \State $S_i \leftarrow S_{i-1} \cup \{m\}$
    \EndFor
    \State $j \leftarrow \arg\max_i \left((1-\alpha(i)) f_{\text{prob}}(S_i)\right)$
    \State $S^{\rm pr} \leftarrow S_j$
    \If{$(1-\alpha(j)) f_{\text{prob}}(S_j) < f_{\text{unprobed}}(\emptyset)$}
        \State $S^{\rm pr} \leftarrow \emptyset$
    \EndIf
  \end{algorithmic}
\end{algorithm}

\subsection{Time Complexity Analysis} 
For each of the \( I \) iterations (the probing budget), the algorithm computes the marginal gain in \( f_{\text{prob}} \) for up to \( M \) arms, with each evaluation taking \( O(M) \) time and each time to compute \( f_{\text{prob}} \), which is related to compute a maximum weighted matching needs \( O((MK)^3) \). After the iterations, computing \( f_{\text{unprobed}}(\emptyset) \) using \cite[Algorithm 1]{chen2022online} has a complexity of \( O((MK)^3) \). The total complexity is \( O(I \cdot M \cdot (MK)^3) \) for the iterations, plus \( O((MK)^3) \) for the post-selection computation. If \( f_{\text{prob}} \) requires sample-based estimation for general distribution, the complexity becomes \( O(IMW(MK)^3) \), where $W$ is number of samples.

\section{The Online Setting}

In this section we tackle the online PUCS setting, where the
probability–mass matrix $\bP$, the reward mean matrix $\bmu$, and the true CDF $\{F_{m,k} \}_{m\in[M],k\in[K]}$ of rewards for each arm-play pair are unknown.  
The learner now faces two coupled challenges: (i) deciding where to
probe under the probing overhead, and (ii) continually updating resource and reward estimates—together with their
confidence bounds—so that the offline greedy routine (Alg.~\ref{offline_greedy})
can be invoked each round and the subsequent assignment remains
optimistic‑for‑exploration. We embed the offline algorithm (Alg.~\ref{offline_greedy}) inside
a two‑phase online algorithm (OLPA) and derive regret guarantees under general
reward distributions.

\subsection{Parameters}
We use history results from time step $1$ to $t$ to update our empirical parameters. Specifically, we define the empirical probability mass function $\hat{p}_{m, d}^{(t)} = \sum_{i=1}^{t} \textbf{1}\{D_{m,i} = d\}/t, \forall m \in [M], d \in [D_{\text{max}}]$ 
and 
 $n^{(t)}_{m,k} := \sum_{s=1}^{t}  \mathbf{1}\{R_{i,m,k} \neq \text{null}\}\mathbf{1}\{k\in C_{i,m}\}, \forall m \in [M], k\in [K]$
 and  empirical mean
 $\hat{\mu}^{(t)}_{m,k} \leftarrow \sum_{i=1}^{t} \mathbf{1}\{R_{i,m,k} \neq \text{null}\}R_{i,m,k}\mathbf{1}\{k\in C_{i,m}\} / n^{(t)}_{m,k}, \forall m \in [M],k\in[K]$
 and empirical CDF
 $\hat{F}^{(t)}_{m,k}(\bx)\leftarrow \sum_{i=1}^{t} \mathbf{1}\{R_{i,m,k} \neq \text{null}\}   \mathbf{1} \{ R_{i,m,k}\leq x\}\mathbf{1}\{k\in C_{i,m}\} / n^{(t)}_{m,k}, \forall x\in\mathbb{R}, m \in [M],k\in[K]$


By \cite[Lemma 9]{maillard2017basic}, we construct a confidence interval for empirical mean $\hat{\mu}^{(t)}_{m,k}$ with high probability. By incorporating the confidence interval, the algorithm can balance exploration and exploitation, ensuring that it does not overly rely on potentially inaccurate estimates of the rewards.
\begin{lemma}
    For any $m \in [M]$ and any $k \in [K]$, it holds that
    \begin{align*}
        \mathbb{P}[\forall t, {\mu}_{m,k} - \hat{\mu}^{(t)}_{m,k} \geq \epsilon^{(t)}_{m,k}] \leq \delta,
    \end{align*}
    where $\delta \in (0, 1)$,

$\epsilon^{(t)}_{m,k}= \sqrt{\left(1+{n^{(t)}_{m,k}}\right) \frac{\ln \left(\sqrt{n^{(t)}_{m,k}+1} / \delta\right)}{2 \left(n^{(t)}_{m,k}\right)^2}}$ if $n^{(t)}_{m,k}>0$ \\
and $\epsilon^{(t)}_{m,k}=+\infty$ if $n^{(t)}_{m,k}=0$.
\end{lemma}

\subsection{Online Learning for Joint Probing and Assignment (OLPA)}
The contribution of our online algorithm (OLPA) is to incorporate the greedy probing idea from Algorithm \ref{offline_greedy} to implement a two-phase update mechanism for joint probing and assignment. In the probing phase, the agent selects a subset of arms to probe based on the expected information gain within the budget. In the assignment phase, the agent assigns plays to arms optimally based on the realized rewards and resources from the probed arms.

\textbf{Initialization (lines 1-5):} The algorithm initializes estimates for resource probabilities \(\hat{p}_{m,d}^{(t)}\), mean rewards \(\hat{\mu}^{(t)}_{m,k}\), play counts \(n^{(t)}_{m,k}\), confidence bounds \(\epsilon^{(t)}_{m,k}\), and CDFs \(\hat{F}^{(t)}_{m,k}(\bx)\) for all arms and plays.

\textbf{Update Function (lines 6-14):} The \textsc{UpdateEstimates} function refines these estimates based on probing and assignment outcomes, using \(\epsilon^{(t)}_{m,k}\) to balance exploration and exploitation.

\textbf{Probing Phase (lines 16-17):} At each time step \(t\), the algorithm selects a probing set \(S_t^{\rm pr}\) using current estimates, maximizing potential rewards based on \(\hat{\bP}^{(t)}\) and \(\{\hat{F}^{(t)}_{m,k}\}\).

\textbf{Assignment Phase (lines 19-20):} After probing, the algorithm determines the optimal play-arm assignment \(\bC_t\) by combining the probing results with current estimates to maximize expected rewards.

\textbf{Iteration:} The process repeats over the time horizon \(T\), iteratively improving estimates and decisions using accumulated data.

\textbf{Remark on novelty.}  
OLPA departs from existing multi‑play bandit methods in two key aspects:  
(i) it uses a custom UCB assignment rule whose confidence radius
$\epsilon_{m,k}^{(t)}$ is tailored to the joint uncertainty of reward and
remaining resources;  
(ii) it couples that UCB rule with our offline greedy probing routine,
forming a unified two‑phase policy that first selects which arms to probe under probing overhead and then assigns plays based on the optimistic estimates.

\subsection{Theoretical Analysis}

We provide a theoretical analysis of our online algorithm (OLPA), demonstrating that it achieves sublinear regret and near-optimal performance under reasonable assumptions. We have the following theorem.
\begin{theorem}\label{thm: online}
   The regret of Algorithm \ref{alg:online} is bound by $\mathcal{R}_{\text{regret}}(\frac{e-1}{2e-1},T) \leq \frac{\sqrt{2}}{2}M  \delta \left(\ln \frac{K T} { \delta} \right)^2
   +4D_{\max} K $ $\left( \sum_{m\in [M]}  \max\limits_{k\in C_m} \mu_{m,k}   \right) \sqrt{2 \ln \frac{2}{\delta}} \sqrt{T} $.
\end{theorem}
\paragraph{Comparison with the no-probing variant.}

We note that the constant in the \(\sqrt{T}\)-term of Theorem~\ref{thm: online} arises from a loose upper bound used in the proof:  
$\left(1 - \alpha( |S_t|) \right)\left( \sum_{m\in [M]\setminus S_t}  |\bC_m| \right)\cdots $ $ \left( \sum_{m\in [M]\setminus S_t}  \max\limits_{k\in C_m} \mu_{m,k}   \right)\leq  K \left( \sum_{m\in [M]}  \max\limits_{k\in C_m} \mu_{m,k}   \right)$.  
The inequality is tight only when \(S_t=\varnothing\) (probing disabled), so our worst‑case guarantee
coincides with that of a no‑probing algorithm; whenever \(|S_t|>0\) the left‑hand side is strictly
smaller, implying a reduced constant bound while the \(\widetilde O(\sqrt{T}+\ln^2 T)\) rate is preserved. 

Experiments (Section 6) show that this reduced constant leads to clearly lower cumulative
regret, confirming that probing delivers practical improvements even though both variants share
the same asymptotic order. The regret scales linearly with \(K\) and logarithmically with \(M\), 
which is consistent with multi-play bandit settings. The complete proof, incorporating probing and UCB estimate errors, is in the supplementary material.


\paragraph{Lower bound.}
As a complement to the upper regret bound, the regret is lower bounded by \(\Omega(\sqrt{T})\) in the worst-case scenario. The proof, based on a standard two-environment construction, is in the supplementary material and shows that our upper bound is tight up to constant factors.

\ignore{
\red{
\paragraph{Regret guarantees.}
To our knowledge, OLPA is the first algorithm in the PUCS framework to attain a
$\widetilde O(\sqrt T)$ regret despite the extra probing phase.  Moreover,
by carefully selecting which arms to probe, OLPA cuts the leading constant
compared to non-probing baselines (from $8D_{\max}K$ down to $4D_{\max}K$;
Theorem~\ref{thm: online}), effectively trading a small probe overhead for
faster learning.  Empirically (Section 6), OLPA further demonstrates
substantial regret reductions over algorithms that never probe, confirming
the practical value of adaptive information acquisition.

}
}
\begin{algorithm}[!t]
  \caption{Online Learning for Joint Probing and Assignment (OLPA)} 
  \label{alg:online}

  \begin{algorithmic}[1]
    \State $\hat{p}_{m, d}^{(t)} \leftarrow 0, \forall m \in [M], d \in [D_{\text{max}}]$
    \State $\hat{\mu}^{(t)}_{m,k} \leftarrow 0, \forall m \in [M], k \in [K]$
    \State $n^{(t)}_{m,k} \leftarrow 0, \forall m \in [M], k \in [K]$
    \State $\epsilon^{(t)}_{m,k} \leftarrow +\infty, \forall m \in [M], k \in [K]$
    \State $\hat{F}^{(t)}_{m,k}(\bx) \equiv 1, \forall x \in \mathbb{R}, m \in [M], k \in [K]$
    
    \Function{UpdateEstimates}{$M$, $\bX_{t,m}$, $N_{t,m}$}
      \State $\hat{p}_{m, d}^{(t)} \leftarrow \frac{(t-1)\hat{p}_{m, d}^{(t-1)} + \textbf{1}\{D_{m,t} = N_{t,m}\}}{t},$
      \Statex \hspace{\algorithmicindent} $\forall m \in [M], d \in [D_{\text{max}}]$
      
      \State $n^{(t)}_{m,k} \leftarrow n^{(t-1)}_{m,k} + \mathbf{1}\{k \in C_{t,m}\},$
      \Statex \hspace{\algorithmicindent} $\forall m \in [M], k \in [K]$
      
      \State $\hat{\mu}^{(t)}_{m,k} \leftarrow \frac{(n^{(t-1)}_{m,k}\hat{\mu}^{(t-1)}_{m,k} + \bX_{t,m} \mathbf{1}\{k \in C_{t,m}\})}{n^{(t)}_{m,k}},$
      \Statex \hspace{\algorithmicindent} $\forall m \in [M], k \in [K]$
      
      \State $\hat{F}^{(t)}_{m,k}(\bx) \leftarrow \frac{(n^{(t-1)}_{m,k}\hat{F}^{(t-1)}_{m,k}(\bx) + \mathbf{1} \{ R_{t,m,k} \leq \bX_{t,m} \} \mathbf{1}\{k \in C_{t,m}\})}{n^{(t)}_{m,k}},$
      \Statex \hspace{\algorithmicindent} $\forall x \in \mathbb{R}, m \in [M], k \in [K]$
      
      \For{$m \in [M]$}
        \For{$k \in [K]$}
          \If{$n^{(t)}_{m,k} \neq 0$}
              \State $\epsilon^{(t)}_{m,k} \leftarrow \sqrt{\left(1 + n^{(t)}_{m,k}\right) \frac{\ln \left(\sqrt{n^{(t)}_{m,k} + 1} / \delta\right)}{2 \left(n^{(t)}_{m,k}\right)^2}}$
          \EndIf
        \EndFor
      \EndFor
    \EndFunction

    \For {$t = 1, 2, \ldots, T$}
        \State $S_t^{\rm pr} \leftarrow$ Algorithm \ref{offline_greedy} $\left(\hat{\bP}^{(t)},\{\hat{F}^{(t)}_{m,k}\}_{m \in [M], k \in [K]}  \right)$
        \State Probe each arm $m \in S_t^{\rm pr}$ and get $(\bX_{t,m}, N_{t,m})$
        
        \State \textsc{UpdateEstimates}($S_t^{\rm pr}$, $\bX_{t,m}$, $N_{t,m}$)
        
        \State Determine the optimal action profile $\bC_t = \arg\max$
        \Statex \hspace{\algorithmicindent} $\mathcal{R}_t^{\text{total}}(S_t^{\rm pr}, \{N_{t,m}, \bX_{t,m}\}_{m \in S},$
        \Statex \hspace{\algorithmicindent} $\{\hat{\bp}^{(t)}_m, \hat{\bmu}^{(t)}_{m} + \boldsymbol{\epsilon}^{(t)}_{m}\}_{m \notin S}, \bC)$
        
        \State Do assignment by the optimal action profile,
        \Statex \hspace{\algorithmicindent} get arm set $S_{\rm mat}$, $(\bX_{t,m}, N_{t,m})$
        \State \textsc{UpdateEstimates}($S_{\rm mat}$, $\bX_{t,m}$, $N_{t,m}$)
        
    \EndFor
  \end{algorithmic}
\end{algorithm}

\begin{figure*}[t]
  \centering
  \captionsetup{width=\textwidth,justification=raggedright,
                singlelinecheck=false}

  \subfloat[]{%
    \includegraphics[width=0.245\textwidth]{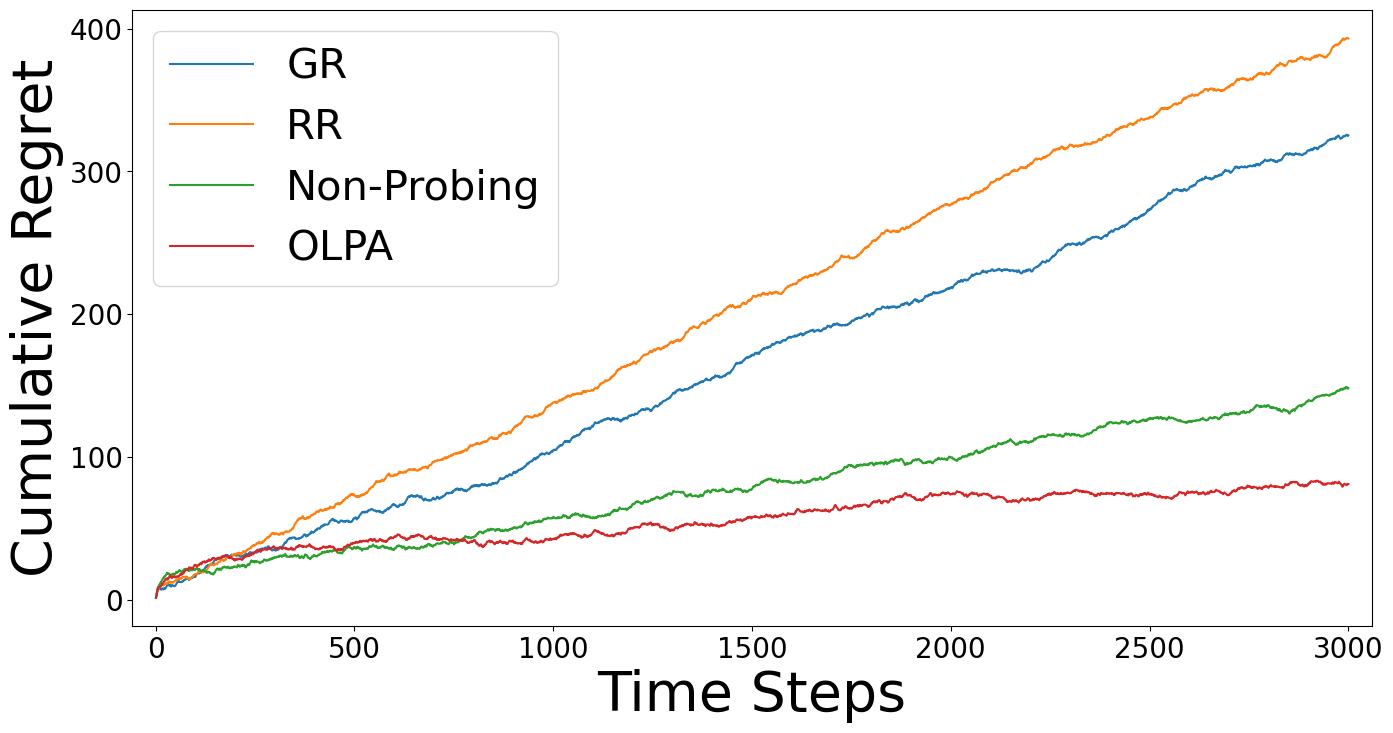}%
    \label{fig1}}%
  \subfloat[]{%
    \includegraphics[width=0.245\textwidth]{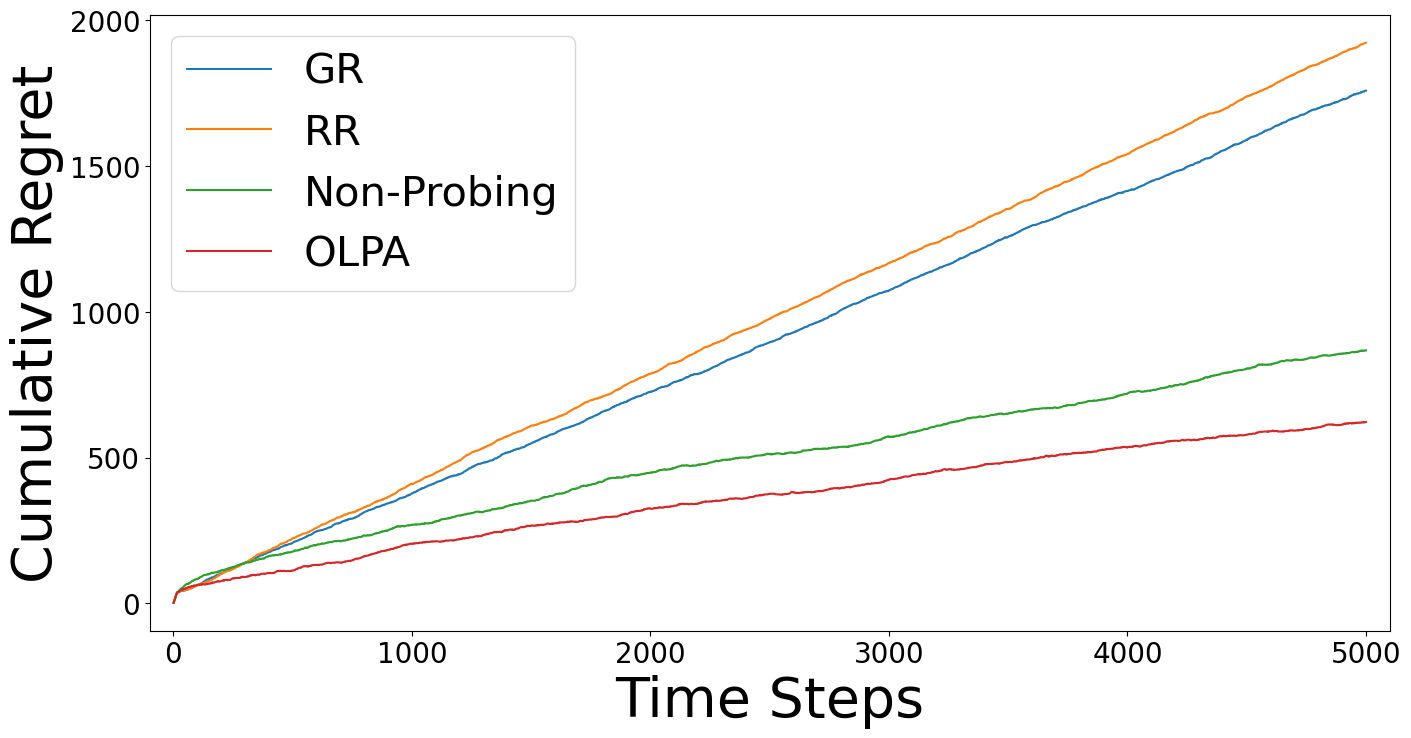}%
    \label{fig2}}%
  \subfloat[]{%
    \includegraphics[width=0.245\textwidth]{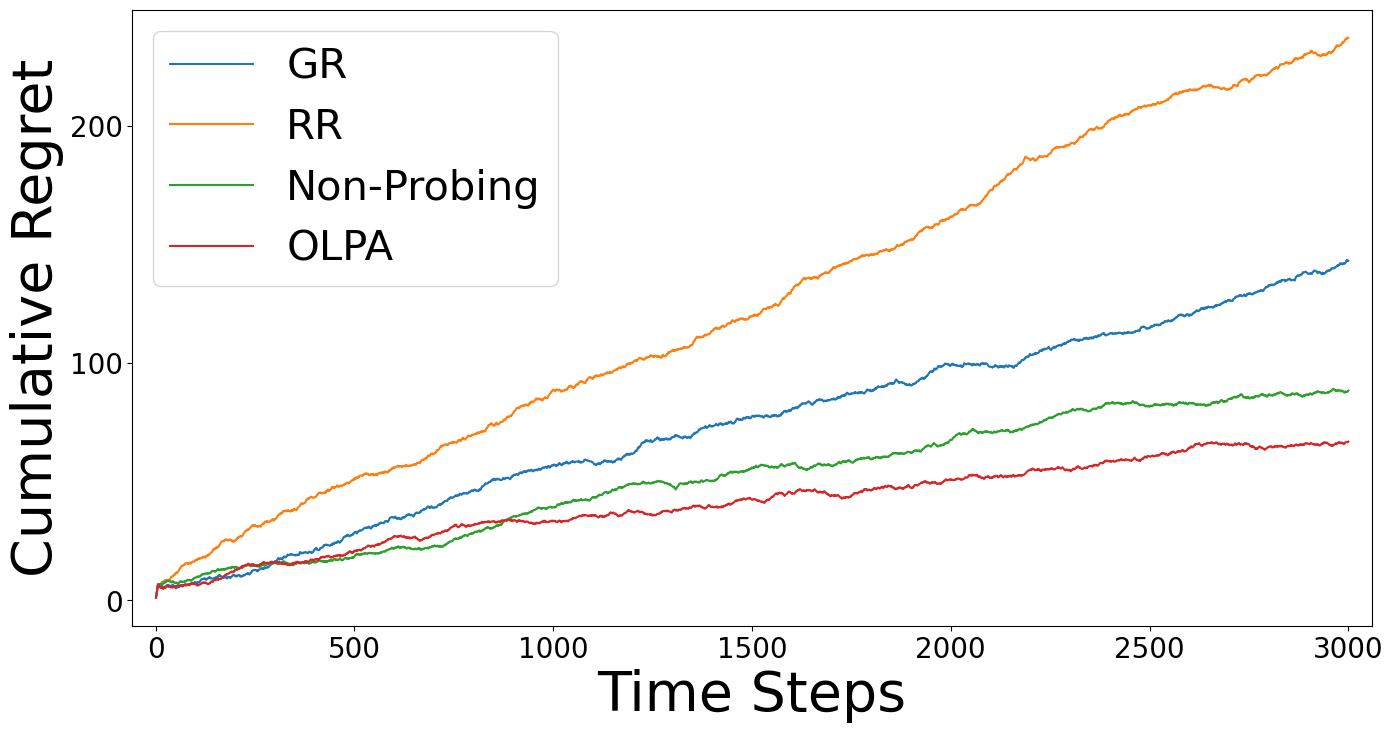}%
    \label{fig3}}%
  \subfloat[]{%
    \includegraphics[width=0.245\textwidth]{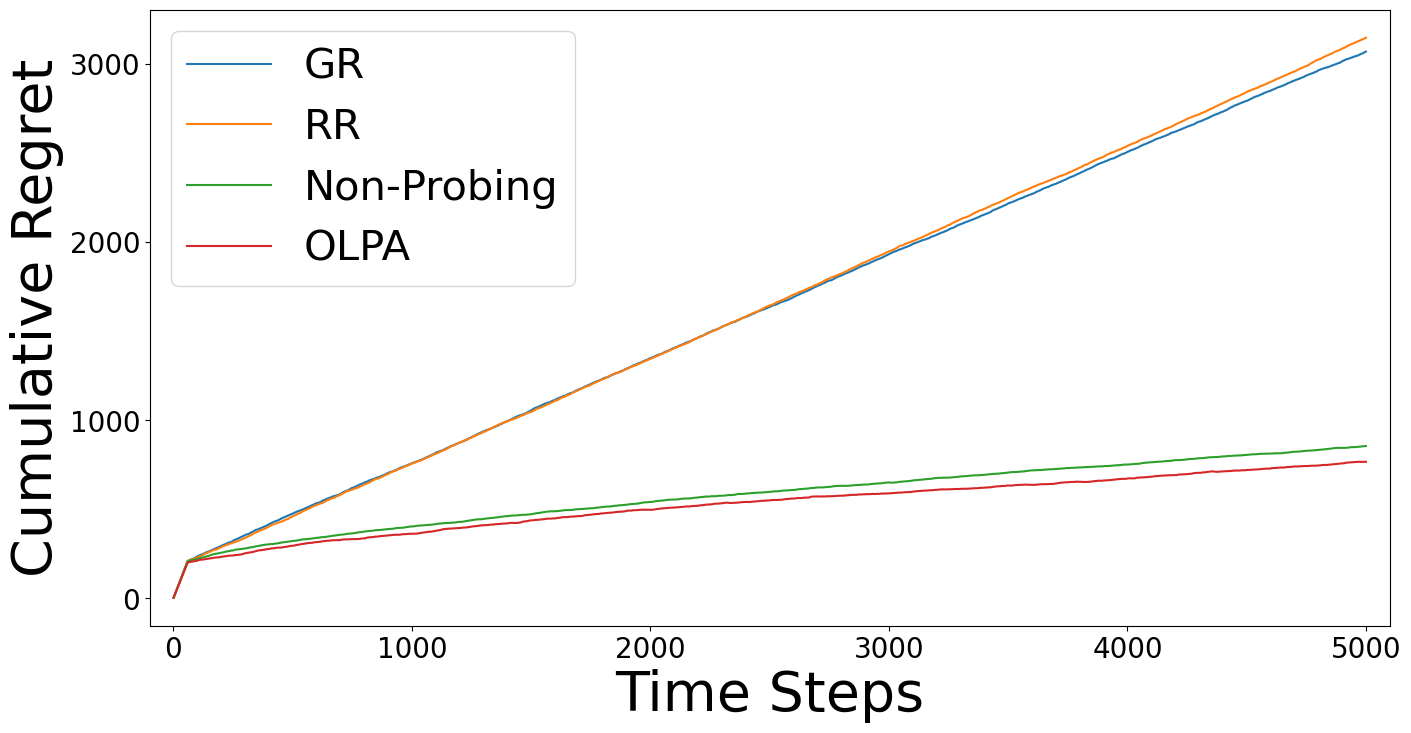}%
    \label{fig4}}%

  \caption{
\small (a): Arms number $M=3$, plays number $K=2$, Bernoulli distribution for reward, $D_{\text{max}}=5$. 
     (b): Arms number $M=5$, plays number $K=3$, Bernoulli distribution for reward, $D_{\text{max}}=7$. 
     (c): Arms number $M=3$, plays number $K=2$, General distribution for reward, $D_{\text{max}}=5$. 
     (d): Arms number $M=10$, plays number $K=6$, General distribution for reward, $D_{\text{max}}=7$. Data from NYYellowTaxi 2016.}
  \label{fig:ablation}
\end{figure*}

\section{Experiments}

In this section, we evaluate our algorithm and compare it with baselines for the ridesharing problem. We utilize two real-world datasets: the NYYellowTaxi 2016 dataset~\cite{shah2020neural} and the Chicago Taxi Trips 2016 dataset~\cite{chauhan2023understanding}. 
The NYYellowTaxi dataset contains trip records, including pickup and dropoff locations, passenger counts, and trip durations in New York City from 3/22/2016 to 3/31/2016. For the Chicago Taxi Trips 2016 dataset, we randomly select a subset of the data spanning from 1/9/2016 to 9/29/2016. These two datasets allow us to evaluate our approach across different urban environments.

Given the similarity between the two datasets, we apply the same data processing method for both. The geographical coordinates of pickup locations are discretized into bins of 0.01 degrees for both latitude and longitude. The frequency of passenger counts within these bins is normalized to derive the probability mass function (PMF), representing the distribution of passengers in each grid cell. Vehicle locations are randomly pre-sampled within the dataset bounds and fixed throughout the experiments. Figure~\ref{fig:ny_environment} visualizes the resulting Manhattan street network (left) and one example environment with pre‑sampled vehicles, pickup requests, and the discretization grid (right).

\begin{figure}[t]
  \centering
\includegraphics[width=0.35\textwidth]{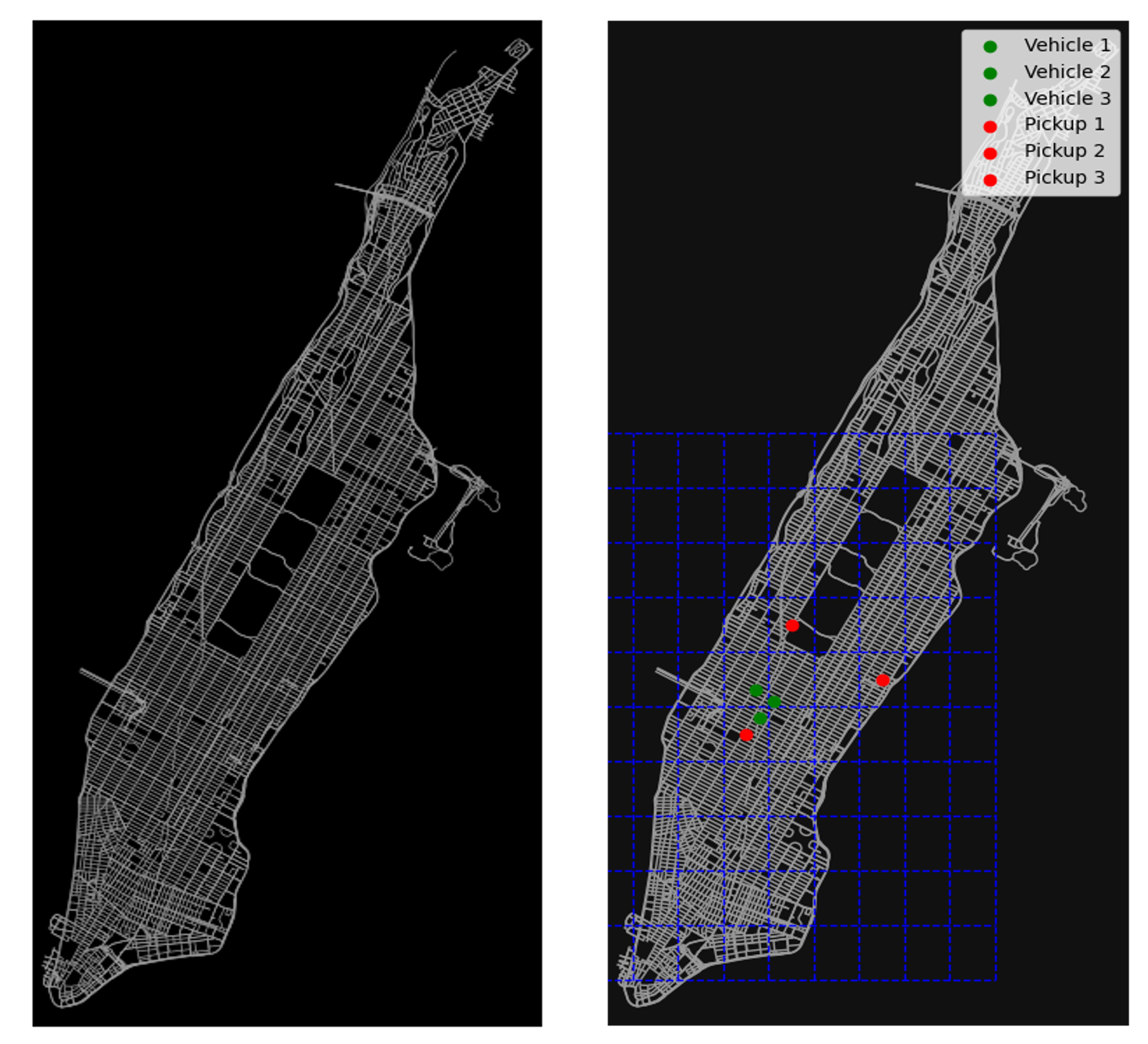}
    \caption{
    Visualization of the NYYellowTaxi dataset and the environment setting.
    Left: Manhattan road network used to sample vehicle and pickup locations.
    Right: an example of environment configuration.
    }
  \label{fig:ny_environment}
\end{figure}

The reward is based on the Manhattan distance between a vehicle and a pickup location. The distances are normalized to the [0, 1] range, where closer distances correspond to higher rewards. We consider two types of reward distributions: (1) a Bernoulli distribution with mean $\mu_{m,k}$, representing the normalized reward for arm $m$ and play $k$, assumed i.i.d. across $t$, $m$, and $k$; and (2) a discrete distribution with support \([0.1, 0.4, 0.7, 1.0]\), where probabilities for each level are derived from the distribution of normalized rewards. Cumulative regret is computed using Equation (\ref{regret}), where \( R(S^{\star}) \) is obtained through exhaustive search. Specific details are provided in the supplementary material.

\begin{table}[t]
  \centering
  \captionsetup{width=\columnwidth,justification=centering}

  \caption{Cumulative regret at 1000/2000/3000 steps on the Chicago Taxi Trips dataset.}
  \label{tab:regret_results}

  \resizebox{0.95\columnwidth}{!}{%
    \begin{tabular}{|c|c|c|c|c|c|c|c|}
      \hline
      \textbf{Set.} & \textbf{M} & \textbf{K} & \textbf{Distr.} & \textbf{Alg.} 
                    & \textbf{1000} & \textbf{2000} & \textbf{3000} \\ \hline
      (a) & 3  & 2 & Bernoulli & OLPA      &  67.67 & 112.09 & 151.40 \\ 
          &    &   &           & GR        & 117.79 & 220.55 & 309.40 \\ 
          &    &   &           & Non–Probing &  87.84 & 122.22 & 170.62 \\ 
          &    &   &           & RR        & 120.16 & 223.61 & 327.28 \\ \hline
      (b) & 5  & 3 & Bernoulli & OLPA      & 133.15 & 225.79 & 292.60 \\ 
          &    &   &           & GR        & 330.30 & 655.82 &1002.47 \\ 
          &    &   &           & Non–Probing & 178.82 & 288.53 & 357.26 \\ 
          &    &   &           & RR        & 292.88 & 580.72 & 831.95 \\ \hline
      (c) & 3  & 2 & General   & OLPA      &  37.71 &  72.86 &  94.27 \\ 
          &    &   &           & GR        & 184.84 & 361.42 & 538.00 \\ 
          &    &   &           & Non–Probing &  54.01 &  89.28 & 118.10 \\ 
          &    &   &           & RR        & 183.80 & 359.63 & 537.37 \\ \hline
      (d) & 10 & 6 & General   & OLPA      &1961.19 &3323.28 &4357.97 \\ 
          &    &   &           & GR        &1865.01 &3546.49 &5254.05 \\ 
          &    &   &           & Non–Probing &1967.15 &3635.66 &5275.73 \\ 
          &    &   &           & RR        &1853.63 &3545.97 &5226.82 \\ \hline
    \end{tabular}
  }
\end{table}


\subsection{Baselines}
We compare OLPA with three baselines.  
\textbf{Non‑Probing} (OnLinActPrf~\cite{chen2022online}) disables probing altogether and solves the
assignment optimally from current estimates.  
\textbf{RR} (Random Probing, Random assignment) first samples a random number $i\!\le\!I$ of arms, probes them uniformly at random,
and then allocates plays by an independent uniform draw—capturing an uninformed exploration
strategy.  
\textbf{GR} (Greedy Probing, Random assignment) uses the same greedy probing strategy as Algorithm \ref{alg:online} in our approach. However, instead of performing optimal assignment, it randomly selects arm-play pairs for assignment after probing.

\subsection{Results}

We compare all algorithms across different settings using the datasets described above.  All four panels of Fig.~\ref{fig:ablation} are generated from the NYYellowTaxi 2016 data. Fig.~\ref{fig1} gives the results when there are 3 locations and 2 vehicles, and the reward distribution is i.i.d.\ Bernoulli. After 800 time steps, our OLPA algorithm significantly outperforms the others, particularly GR and RR, which accumulate much higher regret. In Fig.~\ref{fig2}, we consider a slightly larger setting with 5 locations and 3 vehicles. As the horizon grows, the gap between RR and the other strategies widens, yet OLPA still maintains the lowest regret. Besides the Bernoulli distribution, we also evaluate under the general four‐level reward distribution (Fig.~\ref{fig3}), where OLPA continues to lead, followed by Non-Probing and GR, with RR worst of all. Finally, in the largest setting (Fig.~\ref{fig4}) the performance gap between GR, RR, and the other methods widens even further, yet OLPA still achieves the lowest cumulative regret, demonstrating its robustness across problem scales and reward models. 

Table~\ref{tab:regret_results} reports cumulative regret at 1 k, 2 k, and 3 k steps on the Chicago Taxi Trips dataset. In every configuration OLPA attains the smallest regret—e.g.\ in setting (b) it reduces regret by over 30\%  compared to Non-Probing and by a factor of 3$+$ relative to GR/RR—confirming that probing yields large constant-level gains in practice.


\section{Conclusion}

We introduce PUCS, a probing‑augmented framework for sequential user‑centric selection. In the offline case (known distributions) we give a greedy probing algorithm with a constant‑factor guarantee; in the online case we couple this routine with a UCB‑style assignment to obtain \textsc{OLPA}, whose regret is $\widetilde{O}(\sqrt{T})$ with a strictly smaller leading constant than the no‑probing variant. Experiments on two large taxi‑trip datasets show consistent gains over strong baselines, indicating that PUCS can benefit practical systems such as content recommendation and ridesharing dispatch.

\bibliographystyle{plain}
\bibliography{ecai}

\clearpage
\onecolumn

\appendix

\renewcommand{\thesection}{\arabic{section}}
\setcounter{section}{0}

\setcounter{lemma}{0}
\renewcommand{\thelemma}{\arabic{lemma}}

\setcounter{theorem}{0}
\renewcommand{\thetheorem}{\arabic{theorem}}

\section*{Appendix: Supplementary Material (Proofs and Experiment Details)}
\addcontentsline{toc}{section}{Appendix: Supplementary Material (Proofs)}

\section*{Overview}

In this section, we present the technical proofs that support the results of our paper. The structure is as follows:

\begin{itemize}
    \item \textbf{Section 1}: This section details the Optimal Assignment Policy, which is used to compute \( h_{\text{prob}} \), \( f_{\text{unprobed}} \), and \( h_{\text{total}} \).
    
    \item \textbf{Section 2}: We provide proofs for the offline case. Our \textbf{key contributions} here include Lemma 4, where we demonstrate that \( f_{\text{prob}}(S) \) is submodular, and Theorem \ref{thm: offline}, where we derive an approximation solution for the offline problem.
    
    \item \textbf{Section 3}: This section covers the proofs for the online case. Our \textbf{key contribution} to the online case is the derivation of a regret bound, as established in Theorem \ref{thm:online}.
\end{itemize}

\section{Optimal Assignment with or without Probing}

In this section, we illustrate how to compute \( h_{\text{prob}} \), \( f_{\text{unprobed}} \), and \( h_{\text{total}} \), which are necessary for implementing Algorithms 1 and 2. These computations involve finding the optimal assignment policy based on the probing set and its realizations.

We can conclude the following Observation 1:
\begin{observation}
For computing  \( f_{\text{unprobed}} \), the optimal assignment policy is to find a \(\mathcal{U}\)-saturated and \(\mathcal{V}\)-monotone assignment, where \(\mathcal{U}\) denotes the set of all plays and \(\mathcal{V}\) represents the set of all arms. Given a probing set \( S \) and its realizations, the optimal assignment policy to compute \( h_{\text{prob}} \) is to perform maximum weighted matching. The optimal assignment policy to compute \( h_{\text{total}} \) is to construct a \(\mathcal{U}\)-saturated and \(\mathcal{V}\)-monotone assignment. \label{optimal_matching}
\end{observation}

\subsection{Optimal Assignment without Probing}
When there is no probing, like computing \( f_{\text{unprobed}} \), it is shown in~\cite{chen2022online} that there is a polynomial time algorithm to obtain the optimal assignment in user-centric selection in each time round (using either the accurate reward and resource distributions in the offline case or the estimated distributions in the online setting). 
The optimal assignment policy is to find a \(\mathcal{U}\)-saturated and \(\mathcal{V}\)-monotone assignment~\cite{chen2022online}, where \(\mathcal{U}\) denotes the set of all plays and \(\mathcal{V}\) represents the set of all arms. A \(\mathcal{U}\)-saturated assignment ensures that each play is matched to an arm, while a \(\mathcal{V}\)-monotone assignment ensures that the matching within each arm satisfies the monotonicity property as described in~\cite{chen2022online}.

\subsection{Optimal Assignment with Probing}
When the probing set \( S \) and its realizations \(\{\bX_m,N_{m}\}_{m\in S}\) are fixed, the optimal assignment differs slightly. We now detail how to compute \( h_{\text{prob}} \), \( f_{\text{unprobed}} \), and \( h_{\text{total}} \) under this setting.

\subsubsection{Computing \( h_{\text{prob}} \)}
To evaluate \( h_{\text{prob}}(S, \{\bX_m, N_{m}\}_{m\in S}) \), we construct a bipartite graph \( G = [A \cup B, E, W] \), where:
\begin{itemize}
    \item \( A = \{a_1, \dots, a_K\} \) represents the set of plays.
    \item \( B = \bigcup_{m\in S} \{b_{m,1}, \dots, b_{m,N_m}\} \) represents the set of available resources, where each unit of resource is treated as an independent node.
    \item The edge set \( E = A \times B \) contains edges \((a_k, b_{m,j_m})\) for all \( k \in [K] \), \( m \in S \), and \( j_m \in [N_m] \).
    \item The weight of edge \((a_k, b_{m,j_m})\) is defined as \( w_{(a_k, b_{m,j_m})} = X_{m,k} \), which represents the reward realizations associated with assigning play \( a_k \) to resource \( b_{m,j_m} \).
\end{itemize}

Given this construction, computing \( h_{\text{prob}} \) , which is the value of $\sum_{m\in S} \mathcal{R}_{m}^{\text{prob}}(C_{m}; \bX_{m},N_{m})$, is equivalent to finding the maximum weighted matching in the bipartite graph \( G \). The value of \( h_{\text{prob}}(S, \{\bX_m, N_{m}\}_{m\in S}) \) corresponds to the sum of the weights of the edges in this maximum matching. \( h_{\text{prob}} \) is mainly used in Algorithm 1. By computing the expected value of \( h_{\text{prob}} \), we can get the value of \( f_{\text{prob}} \), which is used in Algorithm 1.

\subsubsection{Computing \( h_{\text{total}} \)}
To evaluate \( h_{\text{total}} \), we adjust the reward expectations and probability distributions based on the probing set \( S \) and the realizations \(\{\bX_m, N_{m}\}_{m\in S}\):
\begin{itemize}
    \item Set \( \bar{p}_{m,d} = 1 \) if \( m \in S \) and \( d = N_m \), and \( \bar{p}_{m,d} = 0 \) if \( m \in S \) and \( d \neq N_m \).
    \item For all \( m \in S \), set \( \bar{\mu}_{m,k} = X_{m,k} \) for all \( k \in [K] \).
\end{itemize}

The optimization problem defining \( h_{\text{total}} \) can now be viewed as the assignment problem defined in \cite[Section 3]{chen2022online}, with:
\begin{itemize}
    \item Reward expectations \( \{\bmu_m\}_{m\in [M]} \) replaced by the adjusted reward expectations \( \{\bar{\bmu}_m\}_{m\in S} \cup \{\bmu_m\}_{m\in [M] \setminus S} \).
    \item Probability mass matrices \( \{\bp_m\}_{m\in [M]} \) replaced by the adjusted matrices \( \{\bar{\bp}_m\}_{m\in S} \cup \{\bp_m\}_{m\in [M] \setminus S} \).
\end{itemize}

By \cite[Theorem 1]{chen2022online}, to get the optimal assignment, one can first construct a \(\mathcal{U}\)-saturated maximum weighted assignment and then transform it to be \(\mathcal{V}\)-monotone. The details of this algorithm are provided in \cite[Algorithm 1]{chen2022online}. Finally, the optimal assignment can be mapped to the optimal action profile as described in \cite[Lemma 3]{chen2022online}. Therefore, computing \( h_{\text{total}} \) is equivalent to evaluating this \(\mathcal{U}\)-saturated and \(\mathcal{V}\)-monotone assignment.

\( h_{\text{total}} \) is used in Algorithm 2 and is mainly used to compute $f(S)$ since $f(S)$ is the expected value of \( h_{\text{total}} \). Then we use $f(S)$ to evaluate $R(S):=(1-\alpha(|S|))f(S)$, and finally compute regret.

\section{Proofs for the offline case}
In this section, we provide detailed proofs related to the offline case of our problem. We start with Proof of Lemma 1 until Theorem \ref{thm: offline}. Lemma 4 and Theorem \ref{thm: offline} are our \textbf{key contributions}. For Lemma 4, we show $f_{\text{prob}}(S)$ is submodular. And in Theorem \ref{thm: offline}, we further obtain an approximation solution to our offline problem.

\begin{lemma}\label{lemma: upper bound for f}
Given probing set $S$, it holds that 
\[f(S)\leq f_{\text{prob}}(S)+ f_{\text{unprobed}}(S).\]
\end{lemma}

\begin{proof}
We first note that
\begin{align*}
 f(S)=&\mathbb{E}_{\substack{X_{t,m,k}\sim R_{t,m,k} \\ N_{t,m}\sim D_{t,m}\\ m\in [M]\setminus S,~i\in[K] }}  \, f(S)\\
 =&\mathbb{E}_{\substack{X_{t,m,k}\sim R_{t,m,k} \\ N_{t,m}\sim D_{t,m}\\ m\in [M]\setminus S,~i\in[K] }}\mathbb{E}_{\substack{X_{t,m,k}\sim R_{t,m,k} \\ N_{t,m}\sim D_{t,m}\\ m\in S,~i\in[K] }}  \, h_{\text{total}}\left(S, \{N_m,\bX_m\}_{m\in S}, \{\bp_m,\bmu_m\}_{m\notin S}\right)\\
 =&\mathbb{E}_{\substack{X_{t,m,k}\sim R_{t,m,k} \\ N_{t,m}\sim D_{t,m}\\ m\in [M],~i\in[K] }}  \, h_{\text{total}}\left(S, \{N_m,\bX_m\}_{m\in S}, \{\bp_m,\bmu_m\}_{m\notin S}\right)
 \end{align*}
where the  first equality holds due to  $f(S)$ doesn't depend on the $\{N_m,\bX_m\}_{m\in [M]\setminus S}$ and the we use the fact that $\{N_m,\bX_m\}_{m\in [M]}$ are independent for the last equality.

Similarly, it holds that
\begin{align*}
    &f_{\text{prob}}(S)=\mathbb{E}_{\substack{X_{t,m,k}\sim R_{t,m,k} \\ N_{t,m}\sim D_{t,m}\\ m\in [M],~i\in[K] }}  \, h_{\text{prob}}\left(S, \{N_m,\bX_m\}_{m\in S}\right),\\
    &f_{\text{unprobed}}(S)=\mathbb{E}_{\substack{X_{t,m,k}\sim R_{t,m,k} \\ N_{t,m}\sim D_{t,m}\\ m\in [M],~i\in[K] }} ~f_{\text{unprobed}}(S).
\end{align*}

{Since above $f(S), f_{\text{unprobed}}(S), f_{\text{prob}}(S)$ share the same expectation, we can focus on the terms that appear afterward.} Then it suffices to show that for any $S$ and any $\{N_m,\bX_m\}_{m\in S}$,  $h_{\text{total }}\leq h_{\text{prob}}+ f_{\text{unprobed}}$.

Let $\bC^{*,\text{total}}=\{C^{*,\text{total}}_m\}_{m\in[M]}$ be the optimizer of the problem defining $h_{\text{total}}$, then 
\begin{align*}
    h_{\text{total}}=  \sum_{m\in M\setminus S}\mathcal{R}_m(C^{*,\text{total}}_m; \bmu_m, \bp_m) +\sum_{m\in S} \mathcal{R}_{m}^{\text{prob}}(C^{*,\text{total}}_m; \bX_{m},N_{m}) 
\end{align*}

Define action profile $\bC^{(1)}=\{C^{(1)}_m\}_{m\in[M]}$ and  $\bC^{(2)}=\{C^{(2)}_m\}_{m\in[M]}$ as follows
\begin{align*}
    C^{(1)}_m:=\begin{cases}
        C^{*,\text{total}}_m,~\text{ if $m\in S$}\\
        \emptyset, ~\text{ if $m\notin S$}
    \end{cases},~~C^{(2)}_m:=\begin{cases}
        C^{*,\text{total}}_m,~\text{ if $m\in [M]\setminus S$}\\
        \emptyset, ~\text{ if $m\in S$}
    \end{cases},
\end{align*}
then $\bC^{(1)}$ and $\bC^{(2)}$ satisfying the constraint of optimization problems defining $h_{\text{prob}}$ and $f_{\text{unprobed}}$, respectively.


Thus we conclude the proof by noting that
\begin{align*}
    h_{\text{prob}}+ f_{\text{unprobed}}\geq& \sum_{m\in S} \mathcal{R}_{m}^{\text{prob}}( C^{(1)}_m; \bX_{m},N_{m})+ \sum_{m\in M\setminus S}\mathcal{R}_m( C^{(2)}_m; \bmu_m, \bp_m) \\
    =& \sum_{m\in S} \mathcal{R}_{m}^{\text{prob}}(C^{*,\text{total}}_m; \bX_{m},N_{m})+\sum_{m\in M\setminus S}\mathcal{R}_m(C^{*,\text{total}}_m; \bmu_m, \bp_m) \\
    =& h_{\text{total}}.
\end{align*}

{The first line of equation
\begin{align*}
    h_{\text{prob}}+ f_{\text{unprobed}} \geq \sum_{m\in S} \mathcal{R}_{m}^{\text{prob}}( C^{(1)}_m; \bX_{m},N_{m})+ \sum_{m\in M\setminus S}\mathcal{R}_m( C^{(2)}_m; \bmu_m, \bp_m).
\end{align*}
holds because $h_{\text{prob}}$ and $f_{\text{unprobed}}$ consider their respective optimal assignments.}

\end{proof}

\begin{lemma}\label{lemma: monotonely decreasing}
$f_{\text{unprobed}}(S)$ is monotonely decreasing, i.e. for any $S\subseteq T\subseteq [M]$,  $f_{\text{unprobed}}(S)\geq  f_{\text{unprobed}}(T) $. 
\end{lemma}

\begin{proof}
    Given $S\subseteq T\subseteq [M]$, we let $\bC^{*,T}=\{C^{*,T}_m\}_{m\in[M]}$ be the optimizer of the problem defining $f_{\text{unprobed}}$, then 
\begin{align*}
    f_{\text{unprobed}}(T)=  \sum_{m\in M\setminus T}\mathcal{R}_m(C^{*,T}_m; \bmu_m, \bp_m).  
\end{align*}

Define action profile $\bC^{(3)}=\{C^{(3)}_m\}_{m\in[M]}$ as follows
\begin{align*}
 C^{(3)}_m:=\begin{cases}
        C^{*,T}_m,~\text{ if $m\in [M]\setminus T$}\\
        \emptyset, ~\text{ if $m\in T$}
    \end{cases},
\end{align*}
then $\bC^{(3)}$  satisfying the constraint of optimization problems defining $ f_{\text{unprobed}}(S,\{\bmu_m,\bp_m\}_{m \in [M]})$, since $[M]\setminus T\subseteq [M]\setminus S$. It follows that
\begin{align*}
    f_{\text{unprobed}}(S)&\geq \sum_{m\in M\setminus S}\mathcal{R}_m(C^{(3)}_m; \bmu_m, \bp_m)\\
    &= \sum_{m\in M\setminus T}\mathcal{R}_m(C^{(3)}_m; \bmu_m, \bp_m)+\sum_{m\in ([M]\setminus S)\cap T}\mathcal{R}_m(C^{(3)}_m; \bmu_m, \bp_m)\\
    &=\sum_{m\in M\setminus T}\mathcal{R}_m(C^{*,T}_m; \bmu_m, \bp_m)+\sum_{m\in ([M]\setminus S)\cap T}\mathcal{R}_m(\emptyset; \bmu_m, \bp_m)\\
    &= h_{\text{unprobed}}(T)
\end{align*}

\end{proof}

\begin{lemma}\label{lemma: monotonely increasing}
$f_{\text{prob}}(S)$ is monotonely increasing, i.e. for any $S\subseteq T\subseteq [M]$,  $f_{\text{prob}}(S)\leq  f_{\text{prob}}(T) $. 
\end{lemma}

\begin{proof}
Using a similar argument in the proof of Lemma \ref{lemma: upper bound for f}, we get that for any $S\subseteq [M]$,
    \begin{align*}
        f_{\text{prob}}(S)=\mathbb{E}_{\substack{X_{t,m,k}\sim R_{t,m,k} \\ N_{t,m}\sim D_{t,m}\\ m\in [M],~i\in[K] }}  \, h_{\text{prob}}\left(S, \{N_m,\bX_m\}_{m\in S}\right).
    \end{align*}
Then it suffices to show that for any $S\subseteq T\subseteq [M]$ and any fixed $\{N_m,\bX_m\}_{m\in [M]}$,  $$ h_{\text{prob}}\left(S, \{N_m,\bX_m\}_{m\in S}\right)\leq h_{\text{prob}}\left(T, \{N_m,\bX_m\}_{m\in T}\right).$$


To begin with, we let $\bC^{*,S}=\{C^{*,S}_m\}_{m\in[M]}$ be the optimizer of the problem defining $h_{\text{prob}}\left(S, \{N_m,\bX_m\}_{m\in S}\right)$, then 
\begin{align*}
    h_{\text{prob}}\left(S, \{N_m,\bX_m\}_{m\in S}\right)=  \sum_{m\in S} \mathcal{R}_{m}^{\text{prob}}(C^{*,S}_m; \bX_{m},N_{m}) .  
\end{align*}

Next, we define action profile $\bC^{(4)}=\{C^{(4)}_m\}_{m\in[M]}$ as follows
\begin{align*}
 C^{(4)}_m:=\begin{cases}
        C^{*,S}_m,~\text{ if $m\in S$}\\
        \emptyset, ~\text{ if $m\notin S$}
    \end{cases},
\end{align*}
then $\bC^{(4)}$  satisfying the constraint of optimization problems defining $ h_{\text{prob}}(T,\{\bmu_m,\bp_m\}_{m \in T})$, since $S\subseteq T$. It follows that
\begin{align*}
   h_{\text{prob}}\left(T, \{N_m,\bX_m\}_{m\in T}\right)&\geq \sum_{m\in T}\mathcal{R}_m(C^{(4)}_m; \bmu_m, \bp_m)\\
    &= \sum_{m\in S}\mathcal{R}_m(C^{(4)}_m; \bmu_m, \bp_m)+\sum_{m\in T\setminus S}\mathcal{R}_m(C^{(4)}_m; \bmu_m, \bp_m)\\
    &=\sum_{m\in S}\mathcal{R}_m(C^{*,S}_m; \bmu_m, \bp_m)+\sum_{m\in T\setminus S}\mathcal{R}_m(\emptyset; \bmu_m, \bp_m)\\
    &= h_{\text{prob}}\left(S, \{N_m,\bX_m\}_{m\in S}\right).
\end{align*}

\end{proof}

\begin{lemma}\label{lemma: submodular}
$f_{\text{prob}}(S)$ is submodular.
\end{lemma}
\begin{proof}


To show $f_{\text{prob}}(S)$ is submodular \cite{nemhauser1978analysis}, we need to show that  for any $S, T\subseteq [M]$,  $f_{\text{prob}}(S)+f_{\text{prob}}(T)\geq f_{\text{prob}}(S\cap T)+f_{\text{prob}}(S\cup T)$ . 

To begin with, we recall that in the proof of Lemma \ref{lemma: upper bound for f}, we have that for any $S\subseteq [M]$,
    \begin{align*}
        f_{\text{prob}}(S)=\mathbb{E}_{\substack{X_{t,m,k}\sim R_{t,m,k} \\ N_{t,m}\sim D_{t,m}\\ m\in [M],~i\in[K] }}  \, h_{\text{prob}}\left(S, \{N_m,\bX_m\}_{m\in S}\right).
    \end{align*}

Since $f_{\text{prob}}$ share the same expectation, we can focus on $h_{\text{prob}}$. 
Then it suffices to show that for any $S, T\subseteq [M]$ and any fixed $\{N_m,\bX_m\}_{m\in [M]}$, 
\begin{align*}
&h_{\text{prob}}\left(S, \{N_m,\bX_m\}_{m\in S}\right)+h_{\text{prob}}\left(T, \{N_m,\bX_m\}_{m\in T}\right)\\
\geq &h_{\text{prob}}\left(S\cap T, \{N_m,\bX_m\}_{m\in S\cap T}\right)+h_{\text{prob}}\left(S\cup T, \{N_m,\bX_m\}_{m\in S\cup T}\right).
\end{align*}

To begin with, we let $\bC^{*,S\cup T}=\{C^{*,S\cup T}_m\}_{m\in[M]}$ and $\bC^{*,S\cap T}=\{C^{*,S\cap T}_m\}_{m\in[M]}$ be the optimizer associated with $h_{\text{probed}}\left(S\cup T, \{N_m,\bX_m\}_{m\in S\cup T}\right)$ and $h_{\text{probed}}\left(S\cap T, \{N_m,\bX_m\}_{m\in S\cap T}\right)$, respectively. Without loss of generality, we assume that $|C^{*,S\cup T}_m|\leq N_m$ and $|C^{*,S\cap T}_m|\leq N_m$. That's because we can always reconstruct such action profile: take optimizer ${\bC}^{*,S\cup T}_m$ as an example, let 
$$\tilde{C}^{*,S\cup T}_m={C}^{*,S\cup T}_m\setminus \{i\in {C}^{*,S\cup T}_m: \text{play $i$ doesn't get resource}\},$$
then $\mathcal{R}_{m}^{\text{prob}}(\tilde{C}^{*,S\cup T}_m; \bX_{m},N_{m})=\mathcal{R}_{m}^{\text{prob}}({C}^{*,S\cup T}_m; \bX_{m},N_{m})$ for all $m$, which implies that $\tilde{\bC}^{*,S\cup T}$ is also a optimizer and $|\tilde{C}^{*,S\cup T}_m|\leq N_m$ for all $m$.

In the rest of the proof, we always assume that $|C^{*,S\cup T}_m|\leq N_m$ and $|C^{*,S\cap T}_m|\leq N_m$, for all $m$.

Next, we construct two action profiles $\bC^S=\{C^S_m\}$ and $\bC^T=\{C^T_m\}$ satisfying the following properties: 
\begin{itemize}
    \item[(i)] $\bC^S$ and $\bC^T$ satisfy the constraint of optimization problems defining $ h_{\text{prob}}(S,\{\bmu_m,\bp_m\}_{m \in S})$ and $ h_{\text{prob}}(T,\{\bmu_m,\bp_m\}_{m \in T})$, respectively;
    \item[(ii)]
    \begin{equation}\label{eq:property 2-formulation 1}      
\begin{aligned}
    &\sum_{m\in S} \mathcal{R}_{m}^{\text{prob}}(C^{S}_m; \bX_{m},N_{m}) +\sum_{m\in T} \mathcal{R}_{m}^{\text{prob}}(C^{T}_m; \bX_{m},N_{m})\\
    = &\sum_{m\in S\cup T} \mathcal{R}_{m}^{\text{prob}}(C^{*,S\cup T}_m; \bX_{m},N_{m}) +\sum_{m\in S\cap T} \mathcal{R}_{m}^{\text{prob}}(C^{*,S\cap T}_m; \bX_{m},N_{m}).
\end{aligned} 
    \end{equation}

\end{itemize}

To begin with, we let
\begin{align*}
    C^S_m=C^{*,S\cup T}_m,~m\in S\setminus T,\\
    C^T_m=C^{*,S\cup T}_m,~m\in T\setminus S,
\end{align*}
then it remains to construct $ C^S_m$ and $ C^T_m$ for $m\in S\cap T$. 

For any $m\in S\cap T$, we split $C^{*,S\cap T}_m$ into three disjoint set:
\begin{align*}
    A_{1,m}=C^{*,S\cap T}_m \cap \left(\bigcup_{m\in S\setminus T}  C^{*,S\cup T}_m\right),\\
    A_{2,m}=C^{*,S\cap T}_m \cap \left(\bigcup_{m\in T\setminus S}  C^{*,S\cup T}_m\right),\\
    A_{3,m}=C^{*,S\cap T}_m \cap \left(\bigcup_{m\in S\cap T}  C^{*,S\cup T}_m\right),\\
    C^{*,S\cap T}_m=A_{1,m} \cup A_{2,m} \cup A_{3,m}.
\end{align*}
Next, for $m\in S\cap T$,  we construct  $ C^S_m$ and $ C^T_m$ as follows:
\begin{align*}
     C^S_m = A_{2,m}\cup C^{*,S\cup T}_m \setminus B_m(|A_{2,m}|+|C^{*,S\cup T}_m|-N_m),\\
     C^T_m = A_{1,m}\cup A_{3,m}\cup B_m(|A_{2,m}|+|C^{*,S\cup T}_m|-N_m).
\end{align*}
where $B_m(i)$ is defined as follows: given any integer $i$,  if $i\geq 0$, we let $B_m(i)$ be any subset of $C^{*,S\cup T}_m\setminus C^{*,S\cap T}_m$ with $|B_m(i)|=i$;  if $i\leq 0$, we let $B_m(i)=\emptyset$. We note that if $|A_{2,m}|+|C^{*,S\cup T}_m|-N_m>0$, it holds that
     \begin{align*}
     | A_{2,m}|+ |C^{*,S\cup T}_m|-N_m\leq & | A_{2,m}|+|C^{*,S\cup T}_m\cap C^{*,S\cap T}_m|+| C^{*,S\cup T}_m\setminus C^{*,S\cap T}_m |-N_m\\
     \leq & | A_{2,m}|+| A_{3,m}|+| C^{*,S\cup T}_m\setminus C^{*,S\cap T}_m |-N_m\\
     \leq & | C^{*,S\cup T}_m\setminus C^{*,S\cap T}_m |,
     \end{align*}
     i.e. $C^{*,S\cup T}_m\setminus C^{*,S\cap T}_m$ contains at least $| A_{2,m}|+ |C^{*,S\cup T}_m|-N_m$ elements, which implies the existence of $B_m(|A_{2,m}|+|C^{*,S\cup T}_m|-N_m)$.

Finally, we show that action profiles $\bC^S=\{C^S_m\}$ and $\bC^T=\{C^T_m\}$ satisfy the desired properties (i) and (ii). 

We begin with property (i). For $C^{S}_m$, we need to show that $C^{S}_{m_1}\cap C^{S}_{m_2}=\emptyset$ for all $m_1,m_2\in S$ and $m_1\neq m_2$. 
\begin{itemize}
    \item if  $m_1,m_2\in S\setminus T$, we have that $C^{S}_{m_1}\cap C^{S}_{m_2}=C^{*,S\cup T}_{m_1}\cap C^{*,S\cup T}_{m_2}=\emptyset$;
    \item if  $m_1,m_2\in S\cap T$, we have that 
    \begin{align*}
    C^{S}_{m_1}\cap C^{S}_{m_2}\subseteq & (A_{2,m_1}\cup C^{*,S\cup T}_{m_1})\cap (A_{2,m_2}\cup C^{*,S\cup T}_{m_2})\\
    =&(A_{2,m_1}\cap A_{2,m_2}) \cup (C^{*,S\cup T}_{m_1}\cap C^{*,S\cup T}_{m_2})\\
    &\cup (A_{2,m_1}\cap C^{*,S\cup T}_{m_2}) \cup (C^{*,S\cup T}_{m_1}\cap A_{2,m_2})\\
    =&  (A_{2,m_1}\cap C^{*,S\cup T}_{m_2}) \cup (C^{*,S\cup T}_{m_1}\cap A_{2,m_2})
    \end{align*}
and
\begin{align*}
     A_{2,m_1}\cap C^{*,S\cup T}_{m_2}=&C^{*,S\cap T}_{m_1} \cap \left(\bigcup_{m\in T\setminus S}  C^{*,S\cup T}_m\right)\cap C^{*,S\cup T}_{m_2}\\
     \subseteq & \left(\bigcup_{m\in T\setminus S}  C^{*,S\cup T}_m\right)\cap C^{*,S\cup T}_{m_2}\\
     =&  \bigcup_{m\in T\setminus S} \left(C^{*,S\cup T}_m\cap C^{*,S\cup T}_{m_2}\right)\\
     =& \emptyset, \text{(since $m_2\in S\cap T$ and $m\in T\setminus S$)},
\end{align*}

Similarly, we have $C^{*,S\cup T}_{m_1}\cap A_{2,m_2}=\emptyset$. Hence $C^{S}_{m_1}\cap C^{S}_{m_2}=\emptyset$.

\item if  $m_1\in S\setminus T$ and  $m_2\in S\cup T$, we have that 
  \begin{align*}
    C^{S}_{m_1}\cap C^{S}_{m_2}\subseteq & 
C^{*,S\cup T}_{m_1} \cap(A_{2,m_2}\cup C^{*,S\cup T}_{m_2})\\
    =& (C^{*,S\cup T}_{m_1}\cap C^{*,S\cup T}_{m_2})\cup (C^{*,S\cup T}_{m_1}\cap A_{2,m_2})\\
    =& C^{*,S\cup T}_{m_1}\cap A_{2,m_2}\\
    =&C^{*,S\cup T}_{m_1} \cap  C^{*,S\cup T}_{m_2}\cap \left(\bigcup_{m\in T\setminus S}  C^{*,S\cup T}_m\right)\\
     \subseteq &C^{*,S\cup T}_{m_1} \cap \left(\bigcup_{m\in T\setminus S}  C^{*,S\cup T}_m\right)\\
     =&  \bigcup_{m\in T\setminus S} \left(C^{*,S\cup T}_{m_1}\cap C^{*,S\cup T}_{m}\right)\\
     =& \emptyset, \text{(since $m_1\in S\cap T$ and $m\in T\setminus S$)}.
    \end{align*}
\end{itemize}

Using a similar argument, we can also verify that for   $C^{T}_{m_1}\cap C^{T}_{m_2}=\emptyset$ for all $m_1,m_2\in T$ and $m_1\neq m_2$.

   To verify the property (ii), we rewrite \eqref{eq:property 2-formulation 1} as follows
\begin{equation}\label{eq: property 2-formulation-2}
\begin{aligned}
    &\sum_{m\in S \setminus T} \mathcal{R}_{m}^{\text{prob}}(C^{S}_m; \bX_{m},N_{m}) +\sum_{m\in  T \setminus S} \mathcal{R}_{m}^{\text{prob}}(C^{T}_m; \bX_{m},N_{m})\\
    &+\sum_{m\in S \cap T} \left(\mathcal{R}_{m}^{\text{prob}}(C^{S}_m; \bX_{m},N_{m}) + \mathcal{R}_{m}^{\text{prob}}(C^{T}_m; \bX_{m},N_{m}) \right)\\
    = &\sum_{m\in S\setminus T} \mathcal{R}_{m}^{\text{prob}}(C^{*,S\cup T}_m; \bX_{m},N_{m}) +\sum_{m\in T\setminus S} \mathcal{R}_{m}^{\text{prob}}(C^{*,S\cup T}_m; \bX_{m},N_{m})\\
    &+\sum_{m\in S\cap T} \left(\mathcal{R}_{m}^{\text{prob}}(C^{*,S\cap T}_m ;\bX_{m},N_{m})+\mathcal{R}_{m}^{\text{prob}}(C^{*,S\cup T}_m; \bX_{m},N_{m})\right)
\end{aligned} 
\end{equation}

By the following relationship, we verify \eqref{eq: property 2-formulation-2}. 
\begin{itemize}
    \item for $m\in S\setminus T$, it holds that $ \mathcal{R}_{m}^{\text{prob}}(C^{S}_m; \bX_{m},N_{m}) = \mathcal{R}_{m}^{\text{prob}}(C^{*,S\cup T}_m; \bX_{m},N_{m})$;
    \item for $m\in T\setminus S$, it holds that $ \mathcal{R}_{m}^{\text{prob}}(C^{T}_m; \bX_{m},N_{m}) = \mathcal{R}_{m}^{\text{prob}}(C^{*,S\cup T}_m; \bX_{m},N_{m})$;
    \item  for $m\in  S\cap T$, we note that $|C^S_m|+|C^T_m|=|C^{*,S\cup T}_m|+|C^{*,S\cap T}_m| $.  If $|A_{2,m}|+|C^{*,S\cup T}_m|-N_m\leq 0$, then $|C^S_m|\leq N_m$ and $|C^T_m|\leq |A_{1,m}\cup A_{3,m} |\leq |C^{*,S\cap T}_m|\leq N_m$. If $|A_{2,m}|+|C^{*,S\cup T}_m|-N_m> 0$,  then $|C^S_m|= N_m$. By the fact that $|C^{*,S\cup T}_m|\leq N_m$, $|C^{*,S\cap T}_m|\leq N_m$, it follows that $|C^T_m|\leq N_m$. 
    \begin{align*}
        &\mathcal{R}_{m}^{\text{prob}}(C^{S}_m; \bX_{m},N_{m}) + \mathcal{R}_{m}^{\text{prob}}(C^{T}_m; \bX_{m},N_{m}) \\
        =& \sum_{i\in C^{S}_m} X_{m,i}+\sum_{i\in C^{T}_m} X_{m,i}\\
        =& \left(\sum_{i\in A_{2,m}} +\sum_{ C^{*,S\cup T}_m \setminus B_m } + \sum_{i\in A_{1,m}\cup A_{3,m}}+\sum_{i\in B_m}\right)  X_{m,i}\\
        =& \sum_{i\in C^{*,S\cup T}_m} X_{m,i}+\sum_{i\in C^{*,S\cap T}_m} X_{m,i}\\
        =&\mathcal{R}_{m}^{\text{prob}}(C^{*,S\cap T}_m; \bX_{m},N_{m})+\mathcal{R}_{m}^{\text{prob}}(C^{*,S\cup T}_m; \bX_{m},N_{m}),
    \end{align*}
where in the third line we use the fact that
\begin{align*}
    &A_{2,m}\cap C^{*,S\cup T}_m \setminus B_m(|A_{2,m}|+|C^{*,S\cup T}_m|-N_m)\\
    \subseteq & A_{2,m}\cap C^{*,S\cup T}_m\\
    =&C^{*,S\cap T}_m \cap \left(\bigcup_{j\in T\setminus S}  C^{*,S\cup T}_j\right)\cap C^{*,S\cup T}_m\\
    \subseteq & \left(\bigcup_{j\in T\setminus S}  C^{*,S\cup T}_j\right)\cap C^{*,S\cup T}_m\\
    =& \emptyset,\text{ (Since $m\in S\cap T$)},
\end{align*}
    and 
    \begin{align*}
       &( A_{1,m}\cup A_{3,m})\cap B_m(|A_{2,m}|+|C^{*,S\cup T}_m|-N_m)\\
       =& C^{*,S\cap T}_m\cap\left(\bigcup_{j\in S}  C^{*,S\cup T}_j\right)\cap  B_m(|A_{2,m}|+|C^{*,S\cup T}_m|-N_m)\\
       \subseteq &   C^{*,S\cap T}_m \cap (C^{*,S\cup T}_m\setminus C^{*,S\cap T}_m)=\emptyset,
    \end{align*}
    where in the last line we recall that $B_m(|A_{2,m}|+|C^{*,S\cup T}_m|-N_m)$ is an subset of $C^{*,S\cup T}_m\setminus C^{*,S\cap T}_m$.
\end{itemize}

Since action profiles $\bC^S=\{C^S_m\}$ and $\bC^T=\{C^T_m\}$ satisfy the desired properties (i) and (ii), it follows that 
\begin{align*}
    &h_{\text{prob}}\left(S, \{N_m,\bX_m\}_{m\in S}\right)+h_{\text{prob}}\left(T, \{N_m,\bX_m\}_{m\in T}\right)\\
    \geq&\sum_{m\in S} \mathcal{R}_{m}^{\text{prob}}(C^{S}_m; \bX_{m},N_{m}) +\sum_{m\in T} \mathcal{R}_{m}^{\text{prob}}(C^{T}_m; \bX_{m},N_{m})\\
    = &\sum_{m\in S\cup T} \mathcal{R}_{m}^{\text{prob}}(C^{*,S\cup T}_m; \bX_{m},N_{m}) +\sum_{m\in S\cap T} \mathcal{R}_{m}^{\text{prob}}(C^{*,S\cap T}_m; \bX_{m},N_{m})\\
   = &h_{\text{prob}}\left(S\cap T, \{N_m,\bX_m\}_{m\in S\cap T}\right)+h_{\text{prob}}\left(S\cup T, \{N_m,\bX_m\}_{m\in S\cup T}\right),
\end{align*} 
by which we conclude the proof.
\end{proof}

\begin{lemma}\label{lemma: output of alg}
Let $S_i$ be the $i$-th probing set found by Algorithm 1 line {5}, $\tilde{S}^*$ be the one that maximizes $(1-\alpha(|S_i|)) f\left(S_i\right)$ (Algorithm 1 line {6}),  $S^{\text {pr }}$ be the final output of Algorithm 1 lines {7-9} .  Then it holds that  $f_{\text{unprobed}}(\emptyset)\leq R(S^{\text{pr}})$ and $(1-\alpha(|\tilde{S}^*|))  f_{\text{prob}}(\tilde{S}^*) \leq R(S^{\text{pr}})$.
\end{lemma}

\begin{proof}
{Based on the line 8 of Algorithm 1, we have case 1: } if $(1-\alpha(|\tilde{S}^*|))f_{\text{prob}}(\tilde{S}^*) \geq  f_{\text{unprobed}}(\emptyset)$, then $S^{\text{pr}}=\tilde{S}^*$ and $$R(S^{\text{pr}})=(1-\alpha(|\tilde{S}^*|))f_{\text{prob}}(\tilde{S}^*) \geq  f_{\text{unprobed}}(\emptyset).$$

{Case 2: }If $(1-\alpha(|\tilde{S}^*|))f_{\text{prob}}(\tilde{S}^*) <  f_{\text{unprobed}}(\emptyset)$, it follows that $S^{\text{pr}}=\emptyset$ and
\[R(S^{\text{pr}})=f(\emptyset)=f_{\text{unprobed}}(\emptyset)>(1-\alpha(|\tilde{S}^*|))f_{\text{prob}}(\tilde{S}^*). \]

{By combining these two cases, we have $f_{\text{unprobed}}(\emptyset)\leq R(S^{\text{pr}})$ and $(1-\alpha(|\tilde{S}^*|))  f_{\text{prob}}(\tilde{S}^*) \leq R(S^{\text{pr}})$ and that is about the conclusion of Lemma 5.}
\end{proof}

\begin{theorem}\label{thm: offline}
Let  $S^*:= \arg\max_{S\subseteq [M]} R(S)$. Algorithm 1 outputs a subset $S^{\text {pr }}$ such that $R\left(S^{\text {pr }}\right) \geq \alpha R\left(S^{\star}\right)$ where $\alpha=\frac{e-1}{2 e-1}$.
\label{offalpha1}
\end{theorem}

\begin{proof}
Let $S_i$ be the $i$-th probing set found by Algorithm 1 line  5 , $\tilde{S}^*$ be the one that maximizes $(1-\alpha(|S_i|)) f\left(S_i\right)$ (Algorithm 1 line 6). We have
$$
\begin{aligned}
R\left(S^{\star}\right) & =(1-\alpha(|S^*|)) f(S^*)\\
& \leq(1-\alpha(|S^*|)) (f_{\text{prob}}(S^*)+ f_{\text{unprobed}}(S^*))\\
&\leq (1-\alpha(|S^*|)) f_{\text{prob}}(S^*)+f_{\text{unprobed}}(S^*)\\
&\leq (1-\alpha(|S^*|)) f_{\text{prob}}(S^*)+f_{\text{unprobed}}(\emptyset)\\
&\leq  (1-\alpha(|S^*|)) \frac{e}{e-1} f_{\text{prob}}(S_{|S^*|})+f_{\text{unprobed}}(\emptyset)\\
&=  (1-\alpha(|S_{|S^*|}|)) \frac{e}{e-1} f_{\text{prob}}(S_{|S^*|})+f_{\text{unprobed}}(\emptyset)\\
&\leq  (1-\alpha(|\tilde{S}^*|)) \frac{e}{e-1} f_{\text{prob}}(\tilde{S}^*)+f_{\text{unprobed}}(\emptyset)\\
&\leq \frac{e}{e-1}  R\left(S^{\text {pr }}\right)+ R\left(S^{\text {pr }}\right)\\
&=\frac{2e-1}{e-1}  R\left(S^{\text {pr }}\right)
\end{aligned}
$$
The first inequality follows by Lemma \eqref{lemma: upper bound for f}; the second inequality follows by $1-\alpha(|S^*|)\leq 1$;  the third inequality follows by  and Lemma \ref{lemma: monotonely decreasing}; the fourth inequality follows by \cite{nemhauser1978analysis}, which states that a simple greedy algorithm with an objective function that is monotone and submodular can obtain an approximation factor of an approximation factor of $1-1/e$. Since $f_{\text{prob}}(S)$ is monotonely increasing and submodular shown in Lemma \ref{lemma: monotonely increasing} and Lemma \ref{lemma: submodular} respectively, we have the approximation factor $1-1/e$. 
{The next one is the second equality, and it follows by $|S_i|=i$ for all $0\leq i\leq I$;} the fifth inequality by the fact that  $\tilde{S}^*$  maximizes $(1-\alpha(|S_i|)) f\left(S_i\right)$; the last inequality follows by Lemma \ref{lemma: output of alg}.
\end{proof}

\section{Proofs for the online case}
In this section, we study the online setting where the probability mass matrix $\bP$  and reward mean matrix $\bmu$ are unknown. We start with Proof of Lemma 6 until Theorem 2. Our \textbf{key contribution} to the online case is to prove a regret bound in Theorem 2. Note: Lemma 9 in the supplementary material corresponds to Lemma 6 in the main paper.

To begin with, we recall that the total reward \(\mathcal{R}_t^{\text{total}}\) in time slot $t$ depends on:
\begin{itemize}
    \item the probing set \( S_t \);
\item Realizations \( X_{t,m,k} \) for \( m \in S \) and \( k \in [K] \);
\item Realizations \( N_{t,m} \);
\item Probability \( \bp_m \) for \( m \notin S \);
\item Expected rewards \(\mu_{m,k}\) for \( m \notin S \) and \( k \in [K] \);
\item The action profile $\bC_t$, which determines the assignment policy \( C_{t,m} \) for each arm $m$.
\end{itemize}

Recall that we define 
$$\epsilon^{(t)}_{m,k}= \sqrt{\left(1+{n^{(t)}_{m,k}}\right) \frac{\ln \left(\sqrt{n^{(t)}_{m,k}+1} / \delta\right)}{2 \left(n^{(t)}_{m,k}\right)^2}}$$
{where $\delta \in (0, 1)$ is a small probability parameter that controls the confidence level of the interval, with smaller values of $\delta$ corresponding to higher confidence. And we next define} the empirical probability mass function 
$$
\hat{p}_{m, d}^{(t)} := \frac{\sum_{i=1}^t \mathbf{1}_{\left\{D_{m, i}=d\right\}}}{t}
, ~\forall m \in [M], d \in [D_{\text{max}}],$$
the counter 
 $$n^{(t)}_{m,k} := \sum_{s=1}^{t}  \mathbf{1}\{R_{i,m,k} \neq \text{null}\}\mathbf{1}\{k\in C_{i,m}\}, \forall m \in [M], k\in [K],$$
empirical mean
 $$\hat{\mu}^{(t)}_{m,k} := \sum_{i=1}^{t} \mathbf{1}\{R_{i,m,k} \neq \text{null}\}R_{i,m,k}\mathbf{1}\{k\in C_{i,m}\} / n^{(t)}_{m,k}, \forall m \in [M],k\in[K]$$
 and empirical CDF
 $$\hat{F}^{(t)}_{m,k}(\bx):= \sum_{i=1}^{t} \mathbf{1}\{R_{i,m,k} \neq \text{null}\}   \mathbf{1} \{ R_{i,m,k}\leq x\}\mathbf{1}\{k\in C_{i,m}\} / n^{(t)}_{m,k}, \forall x\in\mathbb{R}, m \in [M],k\in[K].$$

The following lemma gives a bound for $|\hat{p}_{m, d}^{(t)}-{p}_{m, d}|$.

\begin{lemma}\label{lemma: bound for p}
For any $t\in [T]$, $\delta \in(0,1)$ and any $m\in[M]$, it holds that

\begin{align*}
\mathbb{P}\left[\forall d \in[D_{\max }], ~|\hat{p}_{m, d}^{(t)}-{p}_{m, d}|\geq \sqrt{\frac{2}{t}\ln\frac{4}{\delta}} \right] \leq \delta
\end{align*}

\end{lemma}

\begin{proof}
    
 We define the empirical CDF for $D_{m}$ as $\hat{H}_m^{(t)}(x)$: 
\begin{align*}
    \hat{H}_m^{(t)}(x):=\sum_{j=1}^{\lfloor 
x \rfloor}\hat{p}_{m, j}^{(t)}=\frac{\sum_{i=1}^t \mathbf{1}_{\left\{D_{m, i}\leq x\right\}}}{t}.
\end{align*} 

Define $H_m(x)=\sum_{j=1}^{\lfloor 
x \rfloor}{p}_{m, j}$, then $F_m(x)$ is the CDF for $D_{m}$.

By the Dvoretzky-Kiefer-Wolfowitz Inequality\cite{wasserman2004all}, it follows that
\begin{align*}
    & \mathbb{P}\left[\forall d \in[D_{\max }], ~|\hat{p}_{m, d}^{(t)}-{p}_{m, d}|\geq \epsilon \right]\\
    =&\mathbb{P}\left[\forall d \in[D_{\max }], ~\left|\left(\hat{H}_m^{(t)}(d)-\hat{H}_m^{(t)}(d-1)\right)-   \left({H}_m^{(t)}(d)-{H}_m^{(t)}(d-1)\right)\right|\geq \epsilon \right]\\
    =&  \mathbb{P}\left[\forall d \in[D_{\max }], ~\left|\left(\hat{H}_m^{(t)}(d)-{H}_m^{(t)}(d) \right)+   \left({H}_m^{(t)}(d-1)-\hat{H}_m^{(t)}(d-1)\right)\right|\geq \epsilon \right]\\
    \leq &  \mathbb{P}\left[\forall d \in[D_{\max }], ~\left|\hat{H}_m^{(t)}(d)-{H}_m^{(t)}(d) \right|+   \left|{H}_m^{(t)}(d-1)-\hat{H}_m^{(t)}(d-1)\right|\geq \epsilon \right]\\
    \leq & \mathbb{P}\left[\forall d \in[D_{\max }], ~\left|\hat{H}_m^{(t)}(d)-{H}_m^{(t)}(d) \right|\geq \frac{\epsilon}{2} \right]+\mathbb{P}\left[\forall d \in[D_{\max }], ~\left|\hat{H}_m^{(t)}(d-1)-{H}_m^{(t)}(d-1) \right|\geq  \frac{\epsilon}{2} \right]\\
    \leq & 2 \mathbb{P}\left[ \sup_{x\in\mathbb{R}}\left|\hat{H}_m^{(t)}(x)-{H}_m^{(t)}(x) \right|\geq \frac{\epsilon}{2} \right]\\
    \leq & 4e^{-\frac{1}{2}t\epsilon^2}.
\end{align*}
{The second inequality follows by Union Bound. The fourth inequality follows by the Dvoretzky-Kiefer-Wolfowitz Inequality.}

Let $\delta=4e^{-\frac{1}{2}t\epsilon^2}$, then $\epsilon=\sqrt{\frac{2}{t}\ln\frac{4}{\delta}}$ and the proof is completed.

\end{proof}

Next, we bound $ |\mathcal{R}^{\text{total}}(...,\{\bp_m\})- \mathcal{R}^{\text{total}}(...,\{\hat{\bp}_m)\}|$.

\begin{lemma}\label{lemma: bound for the reward with empricial p}
   For any probing set \( S \), any realizations $\{N_m,\bX_m \}_{m\in S}$, any expected rewards $\{\bmu_m\}_{m\notin S}$ and any action profile $\bC$,
 the following holds with probability at least $1-\delta M D_{\max} $,:
\begin{align*}
&|\mathcal{R}^{\text{total}}(S,\{N_m,\bX_m \}_{m\in S}, \{\bp_m,\bmu_m\}_{m\notin S},\bC)- \mathcal{R}^{\text{total}}(S,\{N_m,\bX_m \}_{m\in S}, \{\hat{\bp}_m,\bmu_m\}_{m\notin S},\bC)|\\
\leq & \sqrt{\frac{2}{t}\ln\frac{4}{\delta}} D_{\max} K \left( \sum_{m\in [M]}  \max_{k\in [K]} \mu_{m,k}   \right)
\end{align*}

\end{lemma}

\begin{proof}

\begin{align*}
    &|\mathcal{R}^{\text{total}}(S,\{N_m,\bX_m \}_{m}, \{\bp_m,\bmu_m\}_{m\notin S},\bC)- \mathcal{R}^{\text{total}}(S,\{N_m,\bX_m \}_{m}, \{\hat{\bp}_m,\bmu_m\}_{m\notin S},\bC)|\\
    =&\left(1 - \alpha( |S|) \right)\left| \left(\sum_{m\in S} \mathcal{R}_{m}^{\text{prob}}(C_{m}; \bX_{m},N_{m}) + \sum_{m\in [M]\setminus S}\mathcal{R}_{m}(C_{m}; \bmu_{m}, \bp_m)\right)-\right.\\
   & ~~~~~~~~~~~~~~~~~~\left. \left(\sum_{m\in S} \mathcal{R}_{m}^{\text{prob}}(C_{m}; \bX_{m},N_{m}) + \sum_{m\in [M]\setminus S}\mathcal{R}_{m}(C_{m}; \bmu_{m}, \hat{\bp}_m)\right)  \right|\\
   =& \left(1 - \alpha( |S|) \right) \left|  \sum_{m\in [M]\setminus S}\mathcal{R}_{m}(C_{m}; \bmu_{m}, \bp_m)-\sum_{m\in [M]\setminus S}\mathcal{R}_{m}(C_{m}; \bmu_{m}, \hat{\bp}_m)  \right|\\
   \leq & \left(1 - \alpha( |S|) \right)\sum_{m\in [M]\setminus S}\left|  \mathcal{R}_{m}(C_{m}; \bmu_{m}, \bp_m)-\mathcal{R}_{m}(C_{m}; \bmu_{m}, \hat{\bp}_m)  \right|\\
   =& \left(1 - \alpha( |S|) \right)\sum_{m\in [M]\setminus S}\left|  \sum_{d=1}^{ |C_{m}| } \left( (p_{m,d}-\hat{p}_{m,d})\sum_{i=1}^d \bmu_{m,i}^{\text{sort}}\right) + \sum_{d=|C_{m}|+1}^{D_{\max}}  (p_{m,d}-\hat{p}_{m,d})\left(\sum_{i=1}^{|C_m|} \bmu_{m,i}^{\text{sort}}\right)  \right|\\
   \leq & \left(1 - \alpha( |S|) \right)\sum_{m\in [M]\setminus S} \left( |\bC_m| D_{\max} \bmu_{m,1}^{\text{sort}}  \sqrt{\frac{2}{t}\ln\frac{4}{\delta}} \right)\\
   =&\sqrt{\frac{2}{t}\ln\frac{4}{\delta}} D_{\max}\left(1 - \alpha( |S|) \right)\sum_{m\in [M]\setminus S} \left( |\bC_m| \max_{k\in C_m} \mu_{m,k}  \right)\\
   \leq & \sqrt{\frac{2}{t}\ln\frac{4}{\delta}} D_{\max}\left(1 - \alpha( |S|) \right)\left( \sum_{m\in [M]\setminus S}  |\bC_m| \right)  \left( \sum_{m\in [M]\setminus S}  \max_{k\in C_m} \mu_{m,k}   \right)\\
   \leq &  \sqrt{\frac{2}{t}\ln\frac{4}{\delta}} D_{\max} K \left( \sum_{m\in [M]}  \max_{k\in C_m} \mu_{m,k}   \right).
\end{align*}

The third equation follows by the definition of $\mathcal{R}_{m}(C_{m}; \bmu_{m}, {\bp}_m)$. The second inequation follows by $\bmu_{m,i}^{\text{sort}}$ definition and Lemma \ref{lemma: bound for p}.
\end{proof}
It is notable that the bound in Lemma~\ref{lemma: bound for the reward with empricial p} doesn't depend on probing set \( S \), or realizations $\{N_m,\bX_m \}_{m\in S}$, or expected rewards $\{\bmu_m\}_{m\notin S}$ or action profile $\bC$.

Next, we bound $ \mathcal{R}^{\text{total}}(..., \{\bar{\bmu}_{t,m}\}) - \mathcal{R}^{\text{total}}(..., \{\bmu_{m}\})$ with respect to $\bar{\bmu}_{m}-\bmu_{m}$.

\begin{lemma}\label{lemma: computation for mu}
     For any probing set \( S \), any realizations $\{N_m,\bX_m \}_{m\in S}$, any probability mass matrix $\bP$, any expected reward $\{\bmu_m\}_{m\in[M]}$ and $\{\bar{\bmu}_m\}_{m\in[M]}$, and  any action profile $\bC$,  the following holds:
    \begin{align*}
        &\mathcal{R}^{\text{total}}(S,\{N_m,\bX_m \}_{m}, \{\bp_m,\bar{\bmu}_m\}_{m\notin S},\bC)- \mathcal{R}^{\text{total}}(S,\{N_m,\bX_m \}_{m}, \{{\bp}_m,\bmu_m\}_{m\notin S},\bC)\\
        =& \left(1 - \alpha( |S|) \right)\sum_{m\in [M]\setminus S}\left[  \sum_{d=1}^{ |C_{m}| } \left( p_{m,d}\sum_{i=1}^d (\bar{\bmu}_{m,i}^{\text{sort}} -\bmu_{m,i}^{\text{sort}})
 \right) + \sum_{d=|C_{m}|+1}^{D_{\max}}  p_{m,d}\left(\sum_{i=1}^{|C_m|} (\bar{\bmu}_{m,i}^{\text{sort}}-\bmu_{m,i}^{\text{sort}})   \right)  \right].
    \end{align*}
\end{lemma}

\begin{proof}
Compute,

\begin{align*}
    &\mathcal{R}^{\text{total}}(S,\{N_m,\bX_m \}_{m}, \{\bp_m,\bar{\bmu}_m\}_{m\notin S},\bC)- \mathcal{R}^{\text{total}}(S,\{N_m,\bX_m \}_{m}, \{{\bp}_m,\bmu_m\}_{m\notin S},\bC)\\
    =&\left(1 - \alpha( |S|) \right)\left[ \left(\sum_{m\in S} \mathcal{R}_{m}^{\text{prob}}(C_{m}; \bX_{m},N_{m}) + \sum_{m\in [M]\setminus S}\mathcal{R}_{m}(C_{m}; \bar{\bmu}_{m}, \bp_m)\right)-\right.\\
   & ~~~~~~~~~~~~~~~~~~\left. \left(\sum_{m\in S} \mathcal{R}_{m}^{\text{prob}}(C_{m}; \bX_{m},N_{m}) + \sum_{m\in [M]\setminus S}\mathcal{R}_{m}(C_{m}; \bmu_{m}, {\bp}_m)\right)  \right]\\
   =& \left(1 - \alpha( |S|) \right) \left[  \sum_{m\in [M]\setminus S}\mathcal{R}_{m}(C_{m}; \bar{\bmu}_{m}, \bp_m)-\sum_{m\in [M]\setminus S}\mathcal{R}_{m}(C_{m}; \bmu_{m}, {\bp}_m)  \right]\\
   = & \left(1 - \alpha( |S|) \right)\sum_{m\in [M]\setminus S}\left[  \mathcal{R}_{m}(C_{m}; \bar{\bmu}_{m}, \bp_m)-\mathcal{R}_{m}(C_{m}; \bmu_{m}, {\bp}_m)  \right]\\
   =& \left(1 - \alpha( |S|) \right)\sum_{m\in [M]\setminus S}\left[  \sum_{d=1}^{ |C_{m}| } \left( p_{m,d}\sum_{i=1}^d (\bar{\bmu}_{m,i}^{\text{sort}} -\bmu_{m,i}^{\text{sort}})
 \right) + \sum_{d=|C_{m}|+1}^{D_{\max}}  p_{m,d}\left(\sum_{i=1}^{|C_m|} (\bar{\bmu}_{m,i}^{\text{sort}}-\bmu_{m,i}^{\text{sort}})   \right)  \right]
\end{align*}
\end{proof}

\begin{corollary}\label{corollary: monotone wrt mu}
     For any probing set \( S \), any realizations $\{N_m,\bX_m \}_{m\in S}$, any probability mass matrix $\bP$, any expected reward $\{\bmu_m\}_{m\in[M]}$ and $\{\bar{\bmu}_m\}_{m\in[M]}$, and  any action profile $\bC$, if $\bar{\mu}_{k,m}\geq {\mu}_{k,m}$ for all $k\in[K]$ and $m\in[M]$, the it holds that
\begin{align*}
        \mathcal{R}^{\text{total}}(S,\{N_m,\bX_m \}_{m}, \{\bp_m,\bar{\bmu}_m\}_{m\notin S},\bC)\geq \mathcal{R}^{\text{total}}(S,\{N_m,\bX_m \}_{m}, \{{\bp}_m,\bmu_m\}_{m\notin S},\bC)
        \end{align*}
         \end{corollary}
\begin{proof}
    For any action profile $\bC$, we have $\bar{\bmu}_{m,i}^{\text{sort}} \geq \bmu_{m,i}^{\text{sort}}$ for all $1\leq i\leq |C_m|$, since $\bar{\mu}_{k,m}\geq {\mu}_{k,m}$ for all $k\in[K]$ and $m\in[M]$. Then by Lemma \ref{lemma: computation for mu}, we complete the proof.
\end{proof}

We recall the following lemma, which gives the bounds for ${\mu}_{m,k}-\hat{\mu}^{(t)}_{m,k}$. 

\begin{lemma}[Corresponding to Lemma 6 in the main paper]
\label{lemma: design of eps}

For any $m\in[M]$ and any $k\in[K]$, it holds that
    \begin{align*}
        \mathbb{P}[\forall t,{\mu}_{m,k}-\hat{\mu}^{(t)}_{m,k}\geq \epsilon^{(t)}_{m,k}]\leq \delta,
    \end{align*}
    where $\delta \in (0, 1)$, $\epsilon^{(t)}_{m,k}= \sqrt{\left(1+{n^{(t)}_{m,k}}\right) \frac{\ln \left(\sqrt{n^{(t)}_{m,k}+1} / \delta\right)}{2 \left(n^{(t)}_{m,k}\right)^2}}$ if $n^{(t)}_{m,k}>0$ and $\epsilon^{(t)}_{m,k}=+\infty$ if $n^{(t)}_{m,k}=0$.
    \end{lemma}
{\begin{proof}

Let \(R_{1,m,k}, R_{2,m,k}, \dots, R_{t,m,k}\) be the sequence of rewards observed for arm \(m\) and play \(k\) up to time \(t\). These rewards are independent and identically distributed (IID) with mean \(\mu_{m,k}\) and are assumed to be sub-Gaussian with parameter \(\sigma\). By  
 in \cite[Lemma 9]{maillard2017basic}, our Lemma 9 is proved. 
\end{proof}
}
\begin{lemma}\label{lemma: bound for reward with UCB mu}
   For any probing set \( S \), any realizations $\{N_m,\bX_m \}_{m\in S}$, any probability mass matrix $\bP$ and any action profile $\bC$,
 the following holds with probability at least $1-\delta M K $:
    \begin{align*}
        &\mathcal{R}^{\text{total}}(S,\{N_{m},\bX_{m} \}_{m\in S}, \{{\bp}_{m},\hat{\bmu}^{(t)}_{m}+\boldsymbol{\epsilon}^{(t)}_{m}\}_{m\notin S},\bC) - \mathcal{R}^{\text{total}}(S,\{N_{m},\bX_{m} \}_{m\in S}, \{\bp_{m},\bmu_{m}\}_{m\notin S},\bC)\\
        \leq& 2\sum_{m\in [M]} \sum_{k\in [K]}\epsilon^{(t)}_{m,k}
    \end{align*}
\end{lemma}

\begin{proof}
Let $\bar{\bmu}_{m,k}^{(t)}=\mu^{(t)}_{m,k}+\epsilon^{(t)}_{m,k}$, then by Lemma \ref{lemma: computation for mu}, it follows that
\begin{align*}
  &\mathcal{R}^{\text{total}}(S,\{N_m,\bX_m \}_{m}, \{\bp_m,\bar{\bmu}_m\}_{m\notin S},\bC)- \mathcal{R}^{\text{total}}(S,\{N_m,\bX_m \}_{m}, \{{\bp}_m,\bmu_m\}_{m\notin S},\bC)\\
        =& \left(1 - \alpha( |S|) \right)\sum_{m\in [M]\setminus S}\left[  \sum_{d=1}^{ |C_{m}| } \left( p_{m,d}\sum_{i=1}^d (\bar{\bmu}_{m,i}^{\text{sort}} -\bmu_{m,i}^{\text{sort}})
 \right) + \sum_{d=|C_{m}|+1}^{D_{\max}}  p_{m,d}\left(\sum_{i=1}^{|C_m|} (\bar{\bmu}_{m,i}^{\text{sort}}-\bmu_{m,i}^{\text{sort}})   \right)  \right]  \\
 \leq& \left(1 - \alpha( |S|) \right)\sum_{m\in [M]\setminus S}\left[  \sum_{d=1}^{ |C_{m}| } \left( p_{m,d}\sum_{i=1}^{|C_m|} (\bar{\bmu}_{m,i}^{\text{sort}} -\bmu_{m,i}^{\text{sort}})
 \right) + \sum_{d=|C_{m}|+1}^{D_{\max}}  p_{m,d}\left(\sum_{i=1}^{|C_m|} (\bar{\bmu}_{m,i}^{\text{sort}}-\bmu_{m,i}^{\text{sort}})   \right)  \right]\\
  =& \left(1 - \alpha( |S|) \right)\sum_{m\in [M]\setminus S}\sum_{i\in C_m} (\bar{\bmu}_{m,i} -\bmu_{m,i})
 \\
    =& \left(1 - \alpha( |S|) \right)\sum_{m\in [M]\setminus S} \sum_{i\in C_m}(\hat{\mu}^{(t)}_{m,i}+\epsilon^{(t)}_{m,k}- \mu_{m,i})
\\
   \leq&  2\sum_{m\in [M]} \sum_{k\in [K]} \epsilon^{(t)}_{m,k},
\end{align*}
{where the second inequation follows by $\sum_{d=1}^{ |C_{m}| }p_{m,d} + \sum_{d=|C_{m}|+1}^{D_{\max}}  p_{m,d} = 1$ and } in the last inequality we use Lemma \ref{lemma: design of eps}.
    
\end{proof}

In the end, we prove Theorem \ref{thm:online}. 


\begin{theorem}\label{thm:online}
   The regret of Algorithm 2 is bounded by:
   \[
   \mathcal{R}_{\text{regret}}\left(\frac{e-1}{2e-1},T\right)\leq 4D_{\max} K \left( \sum_{m\in [M]}  \max_{k\in C_m} \mu_{m,k}   \right) \sqrt{2 \ln \frac{2}{\delta}} \sqrt{T}+\frac{\sqrt{2}}{2}M  \delta \left(\ln \frac{K T} { \delta} \right)^2.
   \]
\end{theorem}

\begin{proof}    
Let $\bar{S}_t^*$ be the output of Algorithm 1 given the true CDF and the true probability mass function, 
{and based on Theorem \ref{thm: offline}, we have $\zeta=\frac{2e-1}{e-1}$.} Then the regret in time slot $t$ can be bounded as follows. 
\begin{align*}
r_{\text{Regret}}(\zeta,t) =& \frac{e-1}{2e-1}R(S_t^*)-R(S_t),\\
\leq &R(\bar{S}_t^*)-R(S_t), \, \text{by Theorem \ref{thm: offline}},\\
=&\mathbb{E}\left[ \max_{\bC_t} \mathcal{R}^{\text{total}}(\bar{S}_t^*,\{N_{t,m},\bX_{t,m} \}_{m\in \bar{S}_t^*}, \{\bp_{t,m},\bmu_{m}\}_{m\notin \bar{S}_t^*},\bC_t)   \right]\\
&-\mathbb{E}\left[ \max_{\bC_t} \mathcal{R}^{\text{total}}(S_t,\{N_{t,m},\bX_{t,m} \}_{m\in S_t}, \{\bp_{t,m},\bmu_{m}\}_{m\notin S_t},\bC_t)   \right], \, \text{(by Definition)},\\
\leq & \mathbb{E}\left[ \max_{\bC_t} \mathcal{R}^{\text{total}}(\bar{S}_t^*,\{N_{t,m},\bX_{t,m} \}_{m\in \bar{S}_t^*}, \{\hat{\bp}_{t,m},\bmu_{m}\}_{m\notin \bar{S}_t^*},\bC_t)   \right]\\
&-\mathbb{E}\left[ \max_{\bC_t} \mathcal{R}^{\text{total}}(S_t,\{N_{t,m},\bX_{t,m} \}_{m\in S_t}, \{\hat{\bp}_{t,m},\bmu_{m}\}_{m\notin S_t},\bC_t)   \right] \\
&+2C_1\sqrt{\frac{2}{t}\ln\frac{4}{\delta}}, \, \text{(by Lemma \ref{lemma: bound for the reward with empricial p})},\\
\leq & \mathbb{E}\left[ \max_{\bC_t} \mathcal{R}^{\text{total}}(\bar{S}_t^*,\{N_{t,m},\bX_{t,m} \}_{m\in \bar{S}_t^*}, \{\hat{\bp}_{t,m},\hat{\bmu}^{(t)}_{m}+\boldsymbol{\epsilon}^{(t)}_m \}_{m\notin \bar{S}_t^*} ,\bC_t)   \right]\\
&-\mathbb{E}\left[ \max_{\bC_t} \mathcal{R}^{\text{total}}(S_t,\{N_{t,m},\bX_{t,m} \}_{m\in S_t}, \{\hat{\bp}_{t,m},\bmu_{m}\}_{m\notin S_t},\bC_t)   \right] \\
&+2C_1\sqrt{\frac{2}{t}\ln\frac{4}{\delta}}, \, \text{(by Corollary \ref{corollary: monotone wrt mu}}),\\
\leq & \mathbb{E}\left[ \max_{\bC_t} \mathcal{R}^{\text{total}}(S_t,\{N_{t,m},\bX_{t,m} \}_{m\in S_t}, \{\hat{\bp}_{t,m},\hat{\bmu}^{(t)}_{m}+\boldsymbol{\epsilon}^{(t)}_m\}_{m\notin S_t},\bC_t)   \right]\\
&-\mathbb{E}\left[ \max_{\bC_t} \mathcal{R}^{\text{total}}(S_t,\{N_{t,m},\bX_{t,m} \}_{m\in S_t}, \{\hat{\bp}_{t,m},\bmu_{m}\}_{m\notin S_t},\bC_t)   \right] \\
&+2C_1\sqrt{\frac{2}{t}\ln\frac{4}{\delta}}, \, \text{(by the optimality of $S_t$)},\\
\leq & 2\sum_{m\in [M]} \sum_{k\in [K]}{\epsilon}^{(t)}_{m,k} + 2C_1\sqrt{\frac{2}{t}\ln\frac{4}{\delta}}, \, \text{(by Lemma \ref{lemma: bound for reward with UCB mu})}.
\end{align*}

where $C_1=D_{\max} K \left( \sum_{m\in [M]}  \max_{k\in C_m} \mu_{m,k}   \right)$.

{The fourth inequality follows from the optimality of \(S_t\), which means that \(S_t\) is the optimal probing set under \(\{ \hat{\bp}_{t,m}, \hat{\bmu}^{(t)}_{m}+\boldsymbol{\epsilon}^{(t)}_m\} \}_{m\notin S_t}\). But $\bar{S}_t^*$ is optimal under the true CDF and the true probability mass function.}

Then by the upper bound of $r_{\text{Regret}}(\zeta,t)$, we have
\begin{align*}
\mathcal{R}_{\text{regret}}(\zeta,T) & =\sum_{t=1}^T r_{\text{Regret}}(\zeta,t) \\
&\leq \sum_{t=1}^T 2C_1\sqrt{\frac{2}{t}\ln\frac{4}{\delta}} +\sum_{t=1}^T \sum_{m \in[M]} \sum_{k\in[K]} 2\sqrt{\left(1+{n^{(t)}_{m,k}}\right) \frac{\ln \left(\sqrt{n^{(t)}_{m,k}+1} / \delta\right)}{2 \left(n^{(t)}_{m,k}\right)^2}}   \\
&\leq  4D_{\max} K \left( \sum_{m\in [M]}  \max_{k\in C_m} \mu_{m,k}   \right) \sqrt{2 \ln \frac{2}{\delta}} \sqrt{T}+\frac{\sqrt{2}}{2}M  \delta \left(\ln \frac{K T} { \delta} \right)^2.
\end{align*}
where in the last inequalities we use the results shown in \cite{chen2022online}:
\begin{align*}
    &\sum_{t=1}^T \sqrt{\frac{1}{2t}\ln\frac{2}{\delta}}\leq \sqrt{2\ln \frac{2}{\delta}}\sqrt{T},\\
    &\sum_{t=1}^T  \sum_{k\in[K]} \sqrt{\left(1+{n^{(t)}_{m,k}}\right) \ln \left(\sqrt{n^{(t)}_{m,k}+1} / \delta\right)}  /{ n^{(t)}_{m,k}}\leq \frac{1}{2}\delta \left(\ln \frac{K T} { \delta} \right)^2.
\end{align*} 
\end{proof}

\paragraph{Comparison with the no--probing bound.}
The \(\sqrt{T}\)-term in Theorem~\ref{thm:online} originates from the following loose inequality
\[
  \bigl(1-\alpha(|S_t|)\bigr)
         \Bigl(\!\sum_{m\in[M]\setminus S_t}\! |C_{t,m}|\Bigr)
         \Bigl(\!\sum_{m\in[M]\setminus S_t}\! \max_{k\in C_{t,m}}\mu_{m,k}\Bigr)
     \;\le\;
     K\!\sum_{m\in[M]}\max_{k\in C_{t,m}}\mu_{m,k},
\]
applied in the proof of Lemma~7.  
The inequality is \emph{tight} only when the probing set is empty (\(|S_t|=0\)); in that special case OLPA reduces to a pure‐assignment routine and its worst‑case constant matches the non‑probing algorithm of~\cite{chen2022online}.  
Whenever \(|S_t|>0\) the left‑hand side is strictly smaller, so the leading constant in the \(\sqrt{T}\) term is \emph{strictly reduced} even though the overall rate remains \(\widetilde{O}(\sqrt{T}+\ln^{2}T)\). Our experimental results corroborate this analysis: OLPA achieves markedly lower cumulative regret than its no‑probing counterpart across all datasets and settings, confirming that probing yield a tangible constant-factor improvement in practice.

As a complement to the upper regret bound, we provide a lower bound for the regret.

\begin{theorem}[Lower Bound for Regret]
For any online algorithm \(\mathcal{A}\), there exists a worst-case configuration of reward distributions such that:
\[
\mathbb{E}[\mathcal{R}_{\text{regret}}] = \Omega(\sqrt{T}).
\]
\end{theorem}

\begin{proof}
Our more general setting (with multiple plays, resource constraints, and probing) 
can be reduced to the single-play multi-armed bandit problem 
by fixing these additional features to trivial values 
(e.g., $K=1$ and no probing). 
Therefore, any lower bound for the single-play setting 
applies to our formulation as well.

We now construct two stochastic environments, $\mathcal{E}_0$ and $\mathcal{E}_1$, 
which any online algorithm $\mathcal{A}$ cannot distinguish within $T$ rounds:

\begin{itemize}
\item \textbf{Environment $\mathcal{E}_0$:} All $M$ arms follow an identical Bernoulli($\mu$) distribution. 
\item \textbf{Environment $\mathcal{E}_1$:} Arm 1 is Bernoulli($\mu$), while the remaining $M-1$ arms are Bernoulli($\mu - \Delta$), where $\Delta = c / \sqrt{T}$ for some constant $c > 0$.
\end{itemize}

To minimize regret, the algorithm must identify whether the environment is $\mathcal{E}_0$ or $\mathcal{E}_1$. However, distinguishing these two environments requires sampling the suboptimal arms sufficiently many times to detect the small reward gap $\Delta$. 

\paragraph{Information-Theoretic Argument.}
Using results from Le Cam’s method and Pinsker’s inequality \cite{Bubeck2012, Lattimore2020}, 
the total variation distance between Bernoulli($\mu$) and Bernoulli($\mu - \Delta$) 
is bounded by $\Delta$. To reliably distinguish between $\mathcal{E}_0$ and $\mathcal{E}_1$, 
the algorithm needs at least $\Omega(1/\Delta^2)$ samples per arm. 
Setting $\Delta = c/\sqrt{T}$ leads to $\Omega(T)$ samples across all arms, 
which is infeasible within $T$ rounds.

\paragraph{Regret Analysis.}
If the algorithm under-samples the suboptimal arms, 
it cannot confidently identify the best arm, resulting in $\Omega(\sqrt{T})$ regret. 
Conversely, if it over-samples suboptimal arms, 
it incurs linear regret due to frequently pulling those arms. 
Hence, for any algorithm $\mathcal{A}$, we have:
\[
\mathbb{E}[\mathcal{R}_{\text{regret}}] 
\;\ge\; 
\frac{1}{2} \Bigl(\mathbb{E}[\mathcal{R}_{\text{regret}} \mid \mathcal{E}_0] 
  + \mathbb{E}[\mathcal{R}_{\text{regret}} \mid \mathcal{E}_1]\Bigr) 
\;=\; 
\Omega(\sqrt{T}).
\]

The worst-case regret under a uniformly mixed environment 
($\mathcal{E}_0$ or $\mathcal{E}_1$ with equal probability) 
is therefore bounded below by $\Omega(\sqrt{T})$, 
matching the known fundamental limits of multi-armed bandits.
\end{proof}

\end{document}